\documentclass{article}
\usepackage{fullpage, url, graphicx, subfigure}
\usepackage{amsmath, amsthm, amssymb}
\usepackage{fancyhdr}
\usepackage{mathrsfs}
\usepackage{epsfig}
\usepackage{graphics}
\usepackage{epsf}
\usepackage{amsmath, amsthm, amssymb, multirow, paralist}
\usepackage{fullpage}
\usepackage{url}
\usepackage{algorithm,algorithmic}
\newtheorem{prop}{Proposition}

\newtheorem{lemma}{Lemma}
\newtheorem{thm}{Theorem}
\newtheorem{rema}[thm]{Remark}

\def \ones {\mathbf{1}}
\def \b {\mathbf{b}}
\def \g {\mathbf{g}}
\def \s {\mathbf{s}}
\def \x {\mathbf{x}}

\def \z {\mathbf{z}}
\def \zt {\mathbf{z}_t}

\def \tht {\mathbf{\theta}_t}

\def \wt {\mathbf{w}_t}
\def \u {\mathbf{u}}

\def \v {\mathbf{v}}
\def \vt {\mathbf{v}_t}
\def \w {\mathbf{w}}

\def \B {\mathcal{B}}
\def \E {\mathbb{E}}
\def \L {\mathcal{L}}
\def \R {\mathbb{R}}
\def \W {\mathcal{W}}
\def \V {\mathcal{V}}

\def \bq {\begin{eqnarray}}
\def \eq {\end{eqnarray}}
\def \bqs {\begin{eqnarray*}}
\def \eqs {\end{eqnarray*}}

\author{Peilin Zhao\footnote{Department of Statistics and Biostatistics, Rutgers University, New Jersey 08854, USA. \emph{peilin.zhao@rutgers.edu}}
        \and Jinwei Yang\footnote{Department of Mathematics, Rutgers University, New Jersey 08854, USA. \emph{yookinwi@math.rutgers.edu}}
        \and Tong Zhang\footnote{Department of Statistics and Biostatistics, Rutgers University, New Jersey 08854, USA. \emph{tzhang@stat.rutgers.edu} }
        \and Ping Li\footnote{Department of Statistics and Biostatistics, Department of Computer Science, Rutgers University, New Jersey 08854, USA. \emph{pingli@stat.rutgers.edu}}
          }

\begin{document}
\date{}
\title{Adaptive Stochastic Alternating Direction Method of Multipliers}
\maketitle

\begin{abstract}
\noindent The  {\em Alternating Direction Method of Multipliers (ADMM)} has been studied for years. The traditional ADMM algorithm needs to compute, at each iteration, an (empirical) expected loss function on all  training examples, resulting in a computational complexity proportional to the number of training examples. To reduce the time complexity, stochastic ADMM algorithms were proposed to replace the expected function with a random loss function associated with one uniformly drawn example plus a Bregman divergence. The Bregman divergence, however, is derived from a simple second order proximal function, the half squared norm, which could be a suboptimal choice.\\

\noindent In this paper, we present a new family of stochastic ADMM algorithms with optimal second order proximal functions, which produce a new family of adaptive subgradient methods. We theoretically prove that their regret bounds are as good as the bounds which could be achieved by the best proximal function that can be chosen in hindsight. Encouraging empirical results on a variety of real-world datasets confirm the effectiveness and efficiency of the proposed algorithms.

\end{abstract}

\section{Introduction}
Originally introduced in~\cite{Glowinski1975,gabay1976dual}, the offline/batch Alternating Direction Method of Multipliers (ADMM)  stemmed from the augmented Lagrangian method, with  its global convergence property  established in~\cite{gabay1983chapter,glowinski1989augmented,eckstein1992douglas}. Recent studies have shown that ADMM  achieves a convergence rate of $O(1/T)$~\cite{DBLP:journals/siamjo/MonteiroS13,he20121} (where $T$ is number of iterations of ADMM), when the objective function is generally convex. Furthermore,  ADMM  enjoys a convergence rate of $O(\alpha^{T})$, for some $\alpha\in(0,1)$, when the objective function is strongly convex and smooth~\cite{luo2012linear,deng2012global}. ADMM has shown attractive performance in a wide range of  real-world problems such as compressed sensing~\cite{yang2011alternating}, image restoration~\cite{goldstein2009split}, video processing, and matrix completion~\cite{DBLP:journals/mp/GoldfarbMS13}, etc.

From the computational perspective, one drawback of ADMM is that, at every iteration, the method needs to compute an (empirical) expected loss function on all the training examples. The computational complexity is propositional to the number of training examples, which makes the original ADMM unsuitable for solving large-scale learning and big data mining problems. The online ADMM (OADMM) algorithm~\cite{wang2012online} was proposed to tackle the computational challenge. For OADMM, the objective function is replaced with an online function  at every step, which only depends on a single training example. OADMM can achieve an average regret bound of $O(1/\sqrt{T})$ for convex objective functions and $O(\log(T)/T)$ for strongly convex objective functions.  Interestingly, although the optimization of the loss function is assumed to be easy in the analysis of~\cite{wang2012online}, it is actually not necessarily easy in practice. To address this issue, the stochastic ADMM algorithm was  proposed, by  linearizing the the online loss function~\cite{ouyang2013stochastic,suzuki2013dual}. In stochastic ADMM algorithms, the online loss function is firstly  uniformly drawn from all the loss functions associated with all the training examples. Then the loss function is replaced with its first order expansion at the current solution plus Bregman divergence from the current solution. The Bregman divergence is based on a simple proximal function, the half squared norm, so that the Bregman divergence is the half squared distance. In this way, the optimization of the loss function enjoys a closed-form solution.  The stochastic ADMM achieves similar convergence rates as OADMM. Using half square norm as proximal function, however, may be a suboptimal choice. Our paper will address this issue.\\

\noindent\textbf{Our contribution}. \  In the previous work~\cite{ouyang2013stochastic,suzuki2013dual} the Bregman divergence is derived from a simple second order function, i.e., the half squared norm, which could be a suboptimal choice~\cite{duchi2011adaptive}. In this paper, we present a new family of stochastic ADMM algorithms with adaptive proximal functions, which can accelerate  stochastic ADMM by using adaptive subgradient. We theoretically prove that the regret bounds  of our methods are as good as those achieved by stochastic ADMM with the best proximal function that can be chosen in hindsight. The effectiveness and efficiency of the proposed algorithms are  confirmed by encouraging empirical evaluations on several real-world datasets.\\

\noindent\textbf{Organization}. Section 2 presents the proposed algorithms.  Section 3 gives our experimental results. Section 4 concludes our paper. Additional proofs can be found in the supplementary material.

\section{Adaptive Stochastic Alternating Direction Method of Multipliers }

\subsection{Problem Formulation}
In this paper, we will study a family of convex optimization problems, where our objective functions are composite. Specially, we are interested in the following equality-constrained optimization task:
\bq\label{eq:origninal-problem}
&&\hspace{-0.4in} \min_{\w\in\W,\v\in\V} f((\w^\top, \v^\top)^\top):= \E_\xi\ell(\w,  \xi) + \varphi(\v),\quad {\rm s.t.}\; A\w+B\v = \b,
\eq
where $\w \in \mathbb{R}^{d_1}$, $\v \in \mathbb{R}^{d_2}$, $A \in \mathbb{R}^{m \times d_1}$, $B \in \mathbb{R}^{m \times d_2}$, $\b \in \mathbb{R}^m$, $\W$ and $\V$ are convex sets. For simplicity, the notation $\ell$ is used for both the instance function value $\ell(\w, \xi)$ and its expectation $\ell(\w) = \E_\xi\ell(\w, \xi)$. It is assumed that  a sequence of identical and independent (i.i.d.) observations  can be drawn from the random vector $\xi$, which satisfies a fixed but unknown distribution. When $\xi$ is deterministic, the above optimization becomes the traditional problem formulation of ADMM~\cite{boyd2011distributed}. In this paper, we will assume the functions $\ell$ and $\varphi$ are convex but not necessarily continuously differentiable. In addition, we denote the optimal solution of \eqref{eq:origninal-problem} as $(\w_{*}^{\top}, \v_{*}^{\top})^{\top}$.

Before presenting the proposed algorithm, we first introduce some notations. For a positive definite matrix $G \in \mathbb{R}^{d_1 \times d_1}$, we define the $G$-norm of a vector $\w$ as $\| \w \|_G: = \sqrt{\w^{\top}G\w}$. When there is no ambiguity, we often use $\| \cdot \|$ to denote the Euclidean norm $\| \cdot\|_2$. We use $\langle \cdot, \cdot\rangle$ to denote the inner product in a finite dimensional Euclidean space. Let $H_t$ be a positive definite matrix for $t \in \mathbb{N}$. Set the proximal function $\phi_t(\cdot)$, as $\phi_t(\w) = \frac{1}{2}\|\w\|^2_{H_t}=\frac{1}{2}\langle \w, H_t\w\rangle$. Then the corresponding {\it Bregman divergence} for $\phi_t(\w)$ is defined as
\bqs
\B_{\phi_t}(\w, \u) = \phi_t(\w) - \phi_t(\u) - \langle \nabla \phi_t(\u), \w-\u \rangle =  \frac{1}{2}\| \w-\u \|^2_{H_t}.
\eqs

\subsection{Algorithm}
To solve the problem~\eqref{eq:origninal-problem}, a popular method is Alternating Direction Multipliers Method (ADMM). ADMM splits the optimizations with respect to $\w$ and $\v$ by minimizing  the augmented Lagrangian:
\bqs
&&\hspace{-0.35in}\min_{\w,\v}\L_\beta(\w,\v,\theta):=\ell(\w)+\varphi(\v) - \langle\theta, A\w+B\v-\b\rangle+ \frac{\beta}{2}\|A\w+B\v-\b\|^2,
\eqs
where $\beta>0$ is a pre-defined penalty. Specifically, the ADMM algorithm  minimizes $\L_{\beta}$ as follows
\bqs
\w_{t+1}=\arg\min_{\w}\L_\beta(\w,\vt,\tht),\ \ \ \v_{t+1}=\arg\min_{\v}\L_\beta(\w_{t+1},\v,\tht),\ \ \ \theta_{t+1}=\tht-\beta(A\w_{t+1}+B\v_{t+1}-\b).
\eqs

At each step, however, ADMM requires calculation of the expectation $\E_\xi\ell(\w, \xi)$, which may be unrealistic or computationally too expensive, since we may only have an unbiased estimate of $\ell(\w)$ or the  expectation $\E_\xi\ell(\w, \xi)$ is an empirical one for big data problem.  To solve this issue, we propose to minimize the its following stochastic approximation:
\bqs
&&\hspace{-0.25in}\L_{\beta, t}(\w, \v, \theta) = \langle\g_t , \w \rangle + \varphi(\v) - \langle \theta,A\w+B\v- \b \rangle + \frac{\beta}{2}\| A\w+B\v-\b\|^2 + \frac{1}{\eta}\B_{\phi_t}(\w,\w_t),
\eqs
where $\g_t=\ell'(\w_t, \xi_{t})$ and $H_t$ for $\phi_t =\frac{1}{2}\|\w\|^2_{H_t}$ will be specified later. This objective linearizes the $\ell(\w,\xi_t)$ and adopts a dynamic  Bregman divergence function to keep the new model near to the previous one. It is easy to see that this proposed approximation includes  the one proposed by~\cite{ouyang2013stochastic} as a special case when $H_t=I$. To minimize the above function, we followed the ADMM algorithm to  optimize over $\w$, $\v$, $\theta$ sequentially, by fixing the others. In addition, we also need to update the $H_t$ for $\B_{\phi_t}$ at every step, which will be specified later. Finally the proposed Adaptive Stochastic Alternating Direction Multipliers Method  ({\bf Ada-SADMM}) is summarized in Algorithm~\ref{alg:adaptive-stochastic-ADMM}.
\begin{algorithm}[htpb]
\caption{ Adaptive Stochastic Alternating Direction Method of Multipliers ({\bf Ada-SADMM}).}\label{alg:adaptive-stochastic-ADMM}
\begin{algorithmic}
\STATE  {\bf Initialize:} $\w_1=\mathbf{0}$, $\u_1=\mathbf{0}$, $\theta_1 = \mathbf{0}$, $H_1=aI$, and $a>0$.
\FOR{$t = 1, 2, \dots, T$}
\STATE Compute $\g_{t}=\ell'(\w_t,\xi_{t})$;
\STATE Update $H_t$ and compute $\B_{\phi_t}$;
\STATE $\w_{t+1} =\arg\min_{\w \in \W}$ $\L_{\beta, t}(\w, \vt, \tht)$;
\STATE $\v_{t+1} =\arg\min_{\v \in \V}$ $\L_{\beta, t}(\w_{t+1}, \v,    \tht)$;
\STATE $\theta_{t+1}=\tht - \beta(A\w_{t+1} + B\v_{t+1} - \b)$;
\ENDFOR
\end{algorithmic}
\end{algorithm}

\subsection{Analysis}
In this subsection we will analyze the performance of the proposed algorithm for general $H_t$, $t=1,\ldots,T$. Specifically, we will provide an expected convergence rate of the iterative solutions. To achieve this goal, we firstly begin with a technical lemma, which will facilitate the later analysis.
\begin{lemma}
Let $\ell(\w,\xi_t)$ and $\varphi(\w)$ be convex functions, and $H_t$ be positive definite, for $t \ge 1$. Then for Algorithm~\ref{alg:adaptive-stochastic-ADMM}, we have the following inequality
\bqs
&&\hspace{-0.3in} \ell(\wt) + \varphi(\v_{t+1}) - \ell(\w) - \varphi(\v) + (\z_{t+1} - \z)^{\top}F(\z_{t+1}) \\
&&\hspace{-0.3in}\leq \frac{\eta\| \g_{t}\|_{H_t^{\ast}}^2}{2} + \frac{1}{\eta}[\B_{\phi_t}(\w_t,\w)-\B_{\phi_t}(\w_{t+1},\w)] + \frac{\beta}{2}(\| A\w + B\vt - \b \|^2 - \| A\w + B \v_{t+1} - \b\|^2)+\langle \delta_{t}, \w - \wt \rangle  \\
&&+ \frac{1}{2\beta}(\| \theta - \tht \|^2 - \| \theta - \theta_{t+1}\|^2),
\eqs
where
$\zt=(\wt^\top, \vt^\top, \tht^\top)^\top$, $\z=(\w^\top, \v^\top, \theta^\top)^\top$, $\mathbf{\delta}_t = \g_t - \ell'(\w_{t})$, and $F(\z) =( (-A^{\top}\theta)^\top, (-B^{\top}\theta)^\top, (A\w+B\v-\b)^\top )^\top$.
\end{lemma}
\begin{proof}
Firstly, using the convexity of $\ell$ and the definition of $\delta_t$, we can obtain
\bqs
&&\hspace{-0.3in}\ell(\wt) - \ell(\w) \leq \langle \ell'(\wt), \wt - \w\rangle = \langle \g_{t}, \w_{t+1} - \w\rangle + \langle \delta_{t}, \w - \wt\rangle + \langle \g_{t}, \wt - \w_{t+1}\rangle.
\eqs
Combining the above inequality with the relation between $\theta_t$ and $\theta_{t+1}$ will derive
\bqs
&&\hspace{-0.3in}\ell(\wt) - \ell(\w) + \langle\w_{t+1} - \w, -A^{\top}\theta_{t+1}\rangle  \\
&&\hspace{-0.3in}\leq \langle\g_t, \w_{t+1} -\w\rangle+ \langle \delta_t, \w-\wt\rangle + \langle\g_t, \wt - \w_{t+1}\rangle +\langle \w_{t+1} - \w, A^{\top}[\beta(A\w_{t+1} + B\v_{t+1} - \b) - \tht]\rangle\nonumber\\
&&\hspace{-0.3in} = \underbrace{\langle\g_t + A^{\top}[\beta(A\w_{t+1} + B\v_{t} - \b) - \tht], \w_{t+1} -\w\rangle}_{L_t}  + \underbrace{\langle \w - \w_{t+1}, \beta A^{\top}B(\vt - \v_{t+1})\rangle}_{M_t}+\langle \delta_t, \w-\wt\rangle \\
&&+\underbrace{ \langle\g_t, \wt - \w_{t+1}\rangle}_{N_t}.
\eqs
To provide an upper bound for the first term $L_t$, taking $D(\u, \v) = \B_{\phi_t}(\u, \v) = \frac{1}{2}\| \u - \v\|^2_{H_t}$ and applying Lemma 1 in \cite{ouyang2013stochastic} to the step of getting $\w_{t+1}$ in the Algorithm 1, we will have
\bqs
&&\hspace{-0.3in} \langle \ell(\wt, \xi_{t}) + A^{\top}[\beta(A\w_{t+1} + B\vt - \b) - \tht], \w_{t+1} - \w\rangle \leq \frac{1}{\eta}[ \B_{\phi_t}(\w_t,\w)-\B_{\phi_t}(\w_{t+1},\w) \hspace{-0.05in} - \B_{\phi_t}(\w_{t+1},\w_t)].
\eqs
To provide an upper bound for the second term $M_t$, we can derive as follows
\bqs
&&\hspace{-0.3in}  \langle \w - \w_{t+1}, \beta A^{\top}B(\vt - \v_{t+1})\rangle = \beta\langle A\w - A\w_{t+1}, B\vt - B\v_{t+1}\rangle\\
&&\hspace{-0.3in} = \frac{\beta}{2}[(\| A\w+B\vt - \b\|^2 - \| A\w + B\v_{t+1} - \b\|^2)+(\| A\w_{t+1} + B\v_{t+1} - \b\|^2- \| A\w_{t+1} + B\vt - \b\|^2)]\\
&&\hspace{-0.3in} \label{eq4}\leq \frac{\beta}{2}(\| A\w + B\vt - \b\|^2 - \| A\w + B\v_{t+1} - \b\|^2) + \frac{1}{2\beta} \| \theta_{t+1} - \tht\|^2.
\eqs
To drive an upper bound for the final term $N_t$, we can use Young's inequality to get
\bqs
&&\langle\g_t, \wt- \w_{t+1}\rangle \leq \frac{\eta \|\g_t\|^2_{H_t^{*}}}{2} + \frac{\|\wt - \w_{t+1}\|^2_{H_t}}{2\eta}=\frac{\eta \|\g_t\|^2_{H_t^{*}}}{2} + \frac{\B_{\phi_t}(\wt, \w_{t+1})}{\eta}.
\eqs
Replacing the terms $L_t$, $M_t$ and $N_t$ with their upper bounds, we will get
\bqs
&&\hspace{-0.3in}\ell(\wt) - \ell(\w) + \langle \w_{t+1} - \w, -A^{\top}\theta_{t+1}\rangle \leq \frac{1}{\eta}[\B_{\phi_t}( \wt,\w) - \B_{\phi_t} \w_{t+1}, \w)] + \frac{\eta \|\g_t\|^2_{H_t^{*}}}{2} + \langle \delta_t, \w-\wt\rangle \\
&&\hspace{2in} + \frac{\beta}{2}(\| A\w + B\vt - \b\|^2 - \| A\w + B\v_{t+1} - \b\|^2) + \frac{1}{2\beta} \| \theta_{t+1} - \tht\|^2.
\eqs
Due to the optimality condition of the step of updating $\v$ in Algorithm 1, i.e., $\partial_\v\L_{\beta,t}(\w_{t+1},\v_{t+1},\theta_t)$ and the convexity of $\varphi$, we have
\bqs
\varphi(\v_{t+1}) - \varphi(\v) + \langle \v_{t+1} - \v, -B^{\top}\theta_{t+1}\rangle \leq 0.
\eqs

Using the fact $A\w_{t+1}+B\v_{t+1}-b=(\theta_t-\theta_{t+1})/\beta$, we have
\bqs
&&\hspace{-0.3in}\langle \theta_{t+1} - \theta,A\w_{t+1} + B\v_{t+1} - \b \rangle  = \frac{1}{2\beta}(\| \theta - \tht\|^2\hspace{-0.05in} - \| \theta - \theta_{t+1}\|^2 \hspace{-0.05in}- \| \theta_{t+1} - \tht\|^2).
\eqs
Combining the above three inequalities and re-arranging the terms will conclude the proof.
\end{proof}
Given the above lemma, now we can analyze the convergence behavior of Algorithm~\ref{alg:adaptive-stochastic-ADMM}. Specifically, we  provide an upper bound on the the objective value and the feasibility violation.
\begin{thm}\label{thm:general-cmd}
Let $\ell(\w,\xi_t)$ and $\varphi(\w)$ be convex functions, and $H_t$ be positive definite, for $t \geq 1$. Then for Algorithm~\ref{alg:adaptive-stochastic-ADMM}, we have the following inequality for any $T \geq 1$ and $\rho > 0$:
\bq\label{eqn:general-convergence}
&&\hspace{-0.3in}\E[f(\bar{\u}_T)-f(\u_*) + \rho\| A\bar{\w}_T + B\bar{\v}_T - \b\|] \nonumber\\
&&\hspace{-0.3in} \leq \frac{1}{2T}\Big(\E\sum_{t = 1}^{T}\big[\frac{2}{\eta}(\B_{\phi_t}(\wt,\w_{*}) - \B_{\phi_t}( \w_{t+1},\w_{*})) +\eta\| \g_{t}\|^2_{H_t^{*}}\big]  + \beta D_{\v_{*}, B}^2 + \frac{\rho^2}{\beta}\Big).
\eq
where $\bar{\u}_T = \left( \frac{1}{T}\sum_{t = 1}^{T}\wt^\top, \frac{1}{T}\sum_{t=2}^{T+1}\vt^\top \right)^\top$, $\u_*= (\w_*^\top, \v_*^\top)^\top$,  and $(\bar{\w}_T, \bar{\v}_T)  =(\frac{1}{T}\sum_{t=2}^{T+1} \wt, \frac{1}{T}\sum_{t=2}^{T+1} \vt)$, and $D_{\v_{*}, B} = \| B \v_{*}\|$.
\end{thm}
\begin{proof}
For convenience, we  denote $\u=(\w^\top,\v^\top)^\top$, $\bar{\theta}_T=\frac{1}{T}\sum^{T+1}_{t=2}\theta_t$, and $\bar{\z}_T=(\bar{\w}_T^\top, \bar{\v}_T^\top, \bar{\theta}_T^\top)^\top$. With these notations, using convexity of $\ell(\w)$ and $\varphi(\v)$ and the monotonicity of operator $F(\cdot)$, we have for any $\z$:
\bqs
f(\bar{\u}_T)- f(\u) + (\bar{\z}_T - \z)^{\top}F(\bar{\z}_T) &\leq &\frac{1}{T}\sum_{t = 1}^T[f((\w_{t}^\top,\v_{t+1}^\top)^\top) - f(\u) + (\z_{t+1} - \z)^{\top}F(\z_{t+1})]\\
&= & \frac{1}{T}\sum_{t = 1}^{T}[\ell(\wt) + \varphi(\v_{t+1}) - \ell(\w) - \varphi(\v) + (\z_{t+1} - \z)^{\top}F(\z_{t+1})].
\eqs

Combining this inequality with Lemma 1 at the optimal solution $(\w, \v) = (\w_{*}, \v_{*})$, we can derive
{\small\begin{align}\notag
&f(\bar{\u}_T)-f(\u_*) + (\bar{\z}_T-\z_*)^\top F(\bar{\z}_T) \\\notag
\leq& \frac{1}{T}\sum_{t=1}^{T}\big\{\frac{1}{\eta}[\B_{\phi_t}(\wt, \w_{*})- \B_{\phi_t}( \w_{t+1}, \w_{*})]+\frac{\eta\| \g_t\|_{H_t^{\ast}}^2}{2} + \langle \delta_t, \w_{*} - \wt \rangle + \frac{\beta}{2}(\| A\w_{*} + B\vt - \b \|^2 \\\notag
&- \| A\w_{*} + B\v_{t+1} - \b\|^2)  + \frac{1}{2\beta}(\| \theta - \tht\|^2 - \| \theta - \theta_{t+1}\|^2)\big\} \\\notag
\leq& \frac{1}{T}\Big\{\sum_{t=1}^{T}\big[\frac{1}{\eta}[\B_{\phi_t}( \wt,\w_{*}) - \B_{\phi_t}(\w_{t+1}, \w_{*})]+\frac{\eta\| \g_t\|_{H_t^{\ast}}^2}{2} +\langle \delta_t, \w_{*}\hspace{-0.05in} - \wt \rangle \big]+ \frac{\beta}{2}\| A\w_{*} + B\v_1\hspace{-0.05in} - \b \|^2 + \frac{1}{2\beta}\| \theta \hspace{-0.05in} - \theta_1\|^2\Big\}\\\notag
 \leq& \frac{1}{T}\Big\{\sum_{t=0}^{T-1}\big[\frac{1}{\eta}(\B_{\phi_t}( \wt,\w_{*}) - \B_{\phi_t}(\w_{t+1}, \w_{*}))+\frac{\eta\| \g_t\|_{H_t^{\ast}}^2}{2} +\langle \delta_t, \w_{*} - \wt \rangle \big] + \frac{\beta}{2}D_{\v_{*}, B}^2
 + \frac{1}{2\beta}\| \theta - \theta_1\|^2\Big\}.
\end{align}}
Because the above inequality is valid for any $\theta$, it also holds in the ball $\mathrm{B}_\rho= \{\theta: \| \theta \| \leq \rho\}$. Combining with the fact that the optimal solution must also be feasible, it follows that
\begin{align}\notag
&\max_{\theta \in \mathrm{B}_\rho}\{f(\bar{\u}_T)-f(\u_*) + (\bar{\z}_T-\z_*)^\top F(\bar{\z}_T)\}\\\notag
=& \max_{\theta \in \mathrm{B}_\rho}\{f(\bar{\u}_T)-f(\u_*)+ \bar{\theta}_T^{\top}(A\w_{*} + B\v_{*} - \b)- \theta^{\top}(A\bar{\w}_T + B\bar{\v}_{T} - \b)\}\\\notag
=& \max_{\theta \in \mathrm{B}_\rho}\{f(\bar{\u}_T) - f(\u_*) - \theta^{\top}(A\bar{\w}_T + B\bar{\v}_{T} - \b)\}=f(\bar{\u}_T)-f(\u_*) + \rho\| A\bar{\w}_T + B\bar{\v}_T - \b\|.
\end{align}
Combining the above two inequalities and taking expectation, we have
\begin{align}\notag
& \E[f(\bar{\u}_T)-f(\u_*) + \rho\| A\bar{\w}_T + B\bar{\v}_T - \b\|]\\\notag
\leq&  \frac{1}{T}\E\Big\{\sum_{t=1}^{T}\big( \frac{1}{\eta}[\B_{\phi_t}(\wt,\w_{*}) - \B_{\phi_t}( \w_{t+1},\w_{*})]+ \frac{\eta\| \g_t\|_{H_t^{\ast}}^2}{2} )+ \langle \delta_t, \w_{*} - \wt \rangle\big) + \frac{\beta}{2}D_{\v_{*}, B}^2 + \frac{1}{2\beta }\| \theta - \theta_1\|^2\Big\}\\\notag
\le& \frac{1}{2T}\Big\{\E\sum_{t = 1}^{T}[\frac{2}{\eta}[\B_{\phi_t}(\wt,\w_{*}) - \B_{\phi_t}( \w_{t+1},\w_{*})]+\eta\| \g_t\|^2_{H_t^{*}}]  + \beta D_{\v_{*}, B}^2 + \frac{\rho^2}{\beta}\Big\},
\end{align}
where we used the fact $\E\delta_t =0$ in the last step. This completes the proof.
\end{proof}

The above theorem allows us to derive regret bounds for a family of algorithms that iteratively modify the proximal functions $\phi_t$ in attempt to lower the regret bounds. Since the rate of convergence is still dependent on $H_t$ and $\eta$, next we are going to choose appropriate positive definite matrix $H_t$ and the constant $\eta$ to optimize the rate of convergence.

\subsection{Diagonal Matrix Proximal Functions}
In this subsection, we restrict  $H_t$ as a diagonal matrix, for two  reasons:  (i) the diagonal matrix will provide results easier to understand than that for the general matrix; (ii) for high dimension problem the general matrix may result in prohibitively expensive computational cost, which is not desirable.

Firstly, we notice that the upper bound in the Theorem~\ref{thm:general-cmd} relies on $\sum^{T}_{t=1}\|\g_{t}\|^2_{H_t^{*}}$. If we assume all the $\g_t$'s are known in advance, we could minimize this term by setting $H_t=diag(\s)$, $\forall t$. We shall use the following proposition.
\begin{prop}
For any $\g_1,\g_2,\ldots,\g_{T}\in \R^{d_1}$, we have
\bqs
\min_{diag(\s)\succeq 0,\ \ones^\top\s \le c}  \sum^{T}_{t=1}\|\g_t\|^2_{diag(\s)}=\frac{1}{c}\big(\sum_{i = 1}^{d_1} \| \g_{1:T, i}\|\big)^2,
\eqs
where $\g_{1:T,i}=(g_{1,i},\ldots,g_{T,i})^\top$ and the minimum is attained at $s_i = c\| \g_{1:T, i}\|/\sum_{j = 1}^{d_1}\| \g_{1:T, j}\|$.
\end{prop}

We omit proof of this proposition, since it is easy to derive. Since we do not have all the $\g_t$'s in advance, we  receives the stochastic (sub)gradients $\g_t$ sequentially instead. As a result, we propose to update the $H_t$ incrementally as:
\bqs
H_t= a I + diag(\s_t),
\eqs
where $s_{t,i}=\|\g_{1:t,i}\|$ and $a\ge 0$. For these $H_t$s, we have the following inequality
\bq\label{eqn:bound-gradient-square-norm}
&&\hspace{-0.25in} \sum^{T}_{t=1}\|\g_{t}\|^2_{H_t^*}=\sum^{T}_{t=1}\langle\g_{t}, (aI+diag(\s_t))^{-1}\g_{t}\rangle  \le \sum^{T}_{t=1}\langle\g_{t}, diag(\s_t)^{-1}\g_{t}\rangle\le 2\sum^{d_1}_{i=1}\|\g_{1:T,i}\|,
\eq
where the last inequality used the Lemma 4 in~\cite{duchi2011adaptive}, which implies this update is a nearly optimal update method for the diagonal matrix case. Finally, the adaptive stochastic  ADMM with diagonal matrix update ({\bf Ada-SADMM$_{diag}$}) is summarized into the Algorithm~\ref{alg:adaptive-stochastic-ADMM-diagonal}.
\begin{algorithm}[htpb]
\caption{ Adaptive Stochastic ADMM with Diagonal Matrix Update ({\bf Ada-SADMM$_{diag}$}).}\label{alg:adaptive-stochastic-ADMM-diagonal}
\begin{algorithmic}
\STATE  {\bf Initialize:} $\w_1= \mathbf{0}$, $\u_1= \mathbf{0}$ , $\theta_1 = \mathbf{0}$, and $a>0$.
\FOR{$t = 1, 2, \dots, T$}
\STATE Compute $\g_t=\ell'(\w_t,\xi_t)$;
\STATE Update $H_t = a I + diag(\s_t),\ {\rm where}\ s_{t,i}=\|\g_{1:t,i}\|$;
\STATE $\w_{t+1} =\arg\min_{\w}$ $\L_{\beta, t}(\w, \vt, \tht)$;
\STATE $\v_{t+1} =\arg\min_{\v \in \V}$ $\L_{\beta, t}(\w_{t+1}, \v,    \tht)$;
\STATE $\theta_{t+1}=\tht - \beta(A\w_{t+1} + B\v_{t+1} - \b)$;
\ENDFOR
\end{algorithmic}
\end{algorithm}

For the convergence rate of the proposed Algorithm 2, we have the following specific theorem.
\begin{thm}
Let $\ell(\w,\xi_t)$ and $\varphi(\w)$ be convex functions for any $t>0$. Then for Algorithm~\ref{alg:adaptive-stochastic-ADMM-diagonal}, we have the following inequality for any $T\ge 1$ and $\rho>0$
\begin{align}\notag
&\E[f(\bar{\u}_T)-f(\u_*)+\rho\|A\bar{\w}_T+B\bar{\v}_T-\b\|] \\\notag
\le& \frac{1}{2T}\Big(\E[ 2\eta\sum^{d_1}_{i=1}\|\g_{1:T,i}\|+\frac{2}{\eta}\max_{t\le T}\|\w_t-\w_*\|^2_{\infty}\sum^{d_1}_{i=1}\|\g_{1:T,i}\| ] + \beta D_{\v_{*}, B}^2 + \frac{\rho^2}{\beta}\Big).
\end{align}
If we further set $\eta=D_{\w,\infty}/\sqrt{2}$ where $D_{\w,\infty}=\max_{\w,\w'}\|\w-\w'\|_{\infty}$, then we have
\bqs
&&\hspace{-0.3in}\E[f(\bar{\u}_T)-f(\u_*)+\rho\|A\bar{\w}_T+B\bar{\v}_T-\b\|] \le\frac{1}{T}\big(\sqrt{2}\E[D_{\w,\infty}\sum^{d_1}_{i=1}\|\g_{1:T,i}\| ] + \frac{\beta}{2} D_{\v_{*}, B}^2 + \frac{\rho^2}{2\beta}\big).
\eqs
\end{thm}
\begin{proof}
We have the following inequality
\bqs
&&\hspace{-0.3in}2\sum_{t = 1}^{T}[\B_{\phi_t}(\wt,\w_{*}) - \B_{\phi_t}( \w_{t+1},\w_{*})]=\sum_{t = 1}^{T}(\| \wt -\w_{*}\|_{H_t}^2 - \| \w_{t+1} -\w_{*}\|_{H_t}^2)\\
&&\hspace{-0.3in}\le\| \w_1 -\w_{*}\|_{H_1}^2+\sum_{t = 1}^{T-1} (\|\w_{t+1}-\w_*\|^2_{H_{t+1}}-\|\w_{t+1}-\w_*\|^2_{H_t})\\
&&\hspace{-0.3in}=\| \w_1 -\w_{*}\|_{H_1}^2+ \sum_{t = 1}^{T-1} \langle\w_{t+1}-\w_*, diag(\s_{t+1}-\s_t)\w_{t+1}-\w_*\rangle\\
&&\hspace{-0.3in}\le\| \w_1 -\w_{*}\|_{H_1}^2+ \sum_{t = 1}^{T-1} \max_i(\w_{t+1,i}-\w_{*,i})^2\|\s_{t+1}-\s_t\|_1\\
&&\hspace{-0.3in}=\| \w_1 -\w_{*}\|_{H_1}^2+ \sum_{t = 1}^{T-1} \|\w_{t+1}-\w_*\|^2_{\infty}(\s_{t+1}-\s_t)^\top\ones\\
&&\hspace{-0.3in}\le\| \w_1 -\w_{*}\|_{H_1}^2+\max_{t\le T}\|\w_t-\w_*\|^2_{\infty}\s_{T}^\top\ones -\|\w_1-\w_*\|^2_{\infty}\s_1^\top\ones\le\max_{t\le T}\|\w_t-\w_*\|^2_{\infty}\sum^{d_1}_{i=1}\|\g_{1:T,i}\|,
\eqs
where the last inequality used  $\langle \s_{T},\ones \rangle=\sum^{d_1}_{i=1}\|\g_{1:T,i}\|$ and $\| \w_1 -\w_{*}\|_{H_1}^2\le\|\w_1-\w_*\|^2_{\infty}\s_1^\top\ones $.

Plugging the above inequality and the inequality~(4) into the inequality~(2), will conclude the first part of the theorem. Then the second part is trivial to be derived.
\end{proof}

\begin{rema}{\rm
For the example of sparse random data, assume that at each round $t$, feature $i$ appears with probability $p_i = {\rm min}\;\{1, ci^{-\alpha}\}$ for some $\alpha \geq 2$ and a constant $c$. Then
\bqs
&& \E[\sum_{i = 1}^d \| g_{1:T, i}\|] = \sum_{i = 1}^d \E[\sqrt{|\{t : |g_{t, i}| = 1\}|}] \leq \sum_{i = 1}^d\sqrt{\E|\{t : |g_{t, i}| = 1\}|} = \sum_{i = 1}^d \sqrt{Tp_i}.
\eqs
In this case, the convergence rate equals $O(\frac{{\rm log}\; d}{\sqrt{T}})$.}
\end{rema}

\subsection{Full Matrix Proximal Functions}
In this subsection, we derive and analyze new updates when we estimate a full matrix $H_t$ for the proximal function instead of a diagonal one. Although full matrix computation may not be attractive for high dimension problems, it may be helpful for tasks with low dimension. Furthermore, it will provide us with a more complete insight. Similar with the analysis for the diagonal case, we first introduce the following proposition (Lemma 15 in~\cite{duchi2011adaptive}).
\begin{prop}\label{p1}
For any $\g_1,\g_2,\ldots,\g_T\in \R^{d_1}$, we have the following inequality
\bqs
\min_{S \succeq 0,\; {\rm tr}(S) \leq c}  \sum^{T}_{t=1}\|\g_t\|^2_{S^{-1}} = \frac{1}{c} tr(G_T)
\eqs
where, $ G_T = \sum_{t=1}^T \g_t\g_t^{\top}$. and  the minimizer is attained at $S = cG_T^{1/2}/{\rm tr}(G_T^{1/2})$. If $G_T$ is not of full rank, then we use its pseudo-inverse to replace its inverse in the minimization problem.
\end{prop}
Because the (sub)gradients are received sequentially, we propose to update the $H_t$ incrementally as
\bqs
H_t=aI+G_t^{\frac{1}{2}},
\eqs
where $G_t = \sum_{i=1}^t \g_i\g_i^{\top}$, $t = 1, \dots, T$. For these $H_t$s, we have the following inequalities
\bq\label{eqn:bound-gradient-norm-with-full-matrix}
&&\hspace{-0.25in}\sum^{T}_{t=1}\|\g_t\|^2_{H_{t}^*}\le \sum_{t=1}^{T}\parallel \g_t\parallel_{S_t^{-1}}^2 \leq 2\sum_{t = 1}^{T}\parallel \g_t\parallel^2_{S_T^{-1}}= 2{\rm tr}(G_T^{1/2}),
\eq
where the last inequality used the Lemma 10 in~\cite{duchi2011adaptive}, which implies this update is a nearly optimal update method for the full matrix case. Finally, the adaptive stochastic  ADMM with full matrix update can be summarized into the Algorithm~\ref{alg:adaptive-stochastic-ADMM-full}.
\begin{algorithm}[htpb]
\caption{ Adaptive Stochastic ADMM with Full Matrix Update ({\bf Ada-SADMM$_{full}$}).}\label{alg:adaptive-stochastic-ADMM-full}
\begin{algorithmic}
\STATE  {\bf Initialize:} $\w_1= \mathbf{0}$, $\u_1= \mathbf{0}$, $\theta_1 = \mathbf{0}$, $G_0= 0$, and $a>0$
\FOR{$t = 1, 2, \dots, T$}
\STATE Compute $\g_t=\ell'(\w_t,\xi_t)$ and update $G_t = G_{t-1} + \g_t\g_t^{\top}$;
\STATE Update $H_t = a I + S_t,\quad {\rm where}\; S_t = G_t^{\frac{1}{2}}$;
\STATE $\w_{t+1} =\arg\min_{\w}$ $\L_{\beta, t}(\w, \vt, \tht)$;
\STATE $\v_{t+1} =\arg\min_{\v \in \V}$ $\L_{\beta, t}(\w_{t+1}, \v,    \tht)$;
\STATE $\theta_{t+1}=\tht - \beta(A\w_{t+1} + B\v_{t+1} - \b)$;
\ENDFOR
\end{algorithmic}
\end{algorithm}

For the convergence rate of the above proposed Algorithm 3, we have the following specific theorem.
\begin{thm}
Let $l(\w, \xi_t)$ and $\varphi(\w)$ are convex functions for any $t>0$. Then for Algorithm \ref{alg:adaptive-stochastic-ADMM-full}, we have the following inequality for any $T \geq 1$, $\rho > 0$,
\begin{align}\notag
&\mathbb{E}[f(\bar{\u}_T) - f(\u_{*}) + \rho \parallel A\bar{\w}_T + B\bar{\v}_T - \b \parallel] \\\notag
\leq& \frac{1}{2T}\big(\mathbb{E}[2\eta {\rm tr}\;(G_T^{1/2}) + \frac{1}{\eta}{\rm max}_{t \leq T}\; \| \w_* - \w_t\|^2{\rm tr}\; (G_T^{\frac{1}{2}})] + \beta D_{\v_{*}, B}^2 + \frac{\rho^2}{\beta} \big).
\end{align}
Furthermore, if we  set $\eta=D_{\w,2}/2$, where $D_{\w,2}=\max_{\w_1,\w_2}\|\w_1-\w_2\|$, then we have
\bqs
&&\hspace{-0.3in}\E[f(\bar{\u}_T)-f(\u_*)+\rho\|A\bar{\w}_T+B\bar{y}_T-\b\|] \le \frac{1}{T}\big(\sqrt{2}\E[D_{\w,2}{\rm tr}\; (G_T^{1/2})] + \frac{\beta}{2} D_{\v_{*}, B}^2 + \frac{\rho^2}{2\beta}\big).
\eqs
\end{thm}
\begin{proof}
We consider the sum of the difference
\bqs
&&\hspace{-0.3in}2\sum_{t = 1}^{T}[\B_{\phi_t}(\wt,\w_{*}) - \B_{\phi_t}( \w_{t+1},\w_{*})]=\sum_{t = 1}^{T}(\| \wt -\w_{*}\|_{H_t}^2 - \| \w_{t+1} -\w_{*}\|_{H_t}^2)\\
&&\hspace{-0.3in}\le\| \w_1 -\w_{*}\|_{H_1}^2 +\sum_{t = 1}^{T-1} (\|\w_{t+1}-\w_*\|^2_{H_{t+1}}-\|\w_{t+1}-\w_*\|^2_{H_t})\\
&&\hspace{-0.3in}=\| \w_1 -\w_{*}\|_{H_1}^2 + \sum_{t = 1}^{T-1} \langle\w_{t+1}-\w_*, (G_{t+1}^{\frac{1}{2}}-G_t^{\frac{1}{2}})(\w_{t+1}-\w_*)\rangle\\
&&\hspace{-0.3in}\le\| \w_1 -\w_{*}\|_{H_1}^2+ \sum_{t = 1}^{T-1}\|\w_{t+1}-\w_*\|^2\lambda_{max}(G_{t+1}^{\frac{1}{2}}-G_t^{\frac{1}{2}}) \\
&&\hspace{-0.3in}=\| \w_1 -\w_{*}\|_{H_1}^2+ \sum_{t = 1}^{T-1} \|\w_{t+1}-\w_*\|^2 tr(G_{t+1}^{\frac{1}{2}}-G_t^{\frac{1}{2}})\\
&&\hspace{-0.3in}\le\| \w_1 -\w_{*}\|_{H_1}^2 + \max_{t\le T-1}\|\w_t-\w_*\|^2 tr(G_T^{\frac{1}{2}})-\|\w_1-\w_*\|^2tr(G_1^{\frac{1}{2}})\le\max_{t\le T}\|\w_t-\w_*\|^2 tr(G_T^{\frac{1}{2}}).
\eqs

Plugging the above inequality and the inequality~(4) into the inequality~(2), will conclude the first part of the theorem. Then the second part is trivial to be derived.
\end{proof}

\section{Experiment}
In this section, we will evaluate the empirical performance of the proposed adaptive stochastic ADMM algorithms for solving GGSVM tasks, which is formulated as the following problem~\cite{ouyang2013stochastic}:
\bqs
&&\hspace{-0.3in}\min_{\w,\v}\frac{1}{n}\sum^n_{i=1}[1-y_i\x_i^\top\w]_++\frac{\gamma}{2}\|\w\|^2 + \nu\|\v\|_1,\quad s.t.\ F\w-\v=0,
\eqs
where $[z]_+=\max(0,z)$ and the matrix $F$ is constructed based on a graph $\mathcal{G} =\{\mathcal{V}, \mathcal{E}\}$. For this graph, $\mathcal{V}=\{w_1,\ldots,w_{d_1}\}$ is a set of variables and $\mathcal{E}=\{e_1,\ldots,e_{|\mathcal{E}|}\}$, where $e_k=\{i,j\}$ is assigned with a weight $\alpha_{ij}$. And the corresponding $F$ is in the form: $F_{ki}=\alpha_{ij}$ and $F_{kj}=-\alpha_{ij}$. To construct a graph for a given dataset, we adopt the sparse inverse covariance estimation~\cite{friedman2008sparse} and determine the sparsity pattern of the inverse covariance matrix $\Sigma^{-1}$. Based on the inverse covariance matrix, we connect all index pairs $(i,j)$ with $\Sigma^{-1}_{ij}\not=0$ and assign $\alpha_{ij}=1$.

\subsection{Experimental Testbed and Setup}
To  examine the performance, we test all the algorithms on 6  real-world datasets from web machine learning repositories, which are listed in the Table~\ref{tab:datasets}. ``news20'' is the ``20 Newsgroups'' downloaded from \footnote{\footnotesize \url{http://www.cs.nyu.edu/~roweis/data.html}}, while the other datasets can be downloaded from LIBSVM website\footnote{\footnotesize \url{http://www.csie.ntu.edu.tw/~cjlin/libsvmtools/datasets}}. For each dataset, we randomly divide it into two folds: training set with $80\%$ of examples and test set with the rest.
\begin{table}[htpb]
\renewcommand*\arraystretch{1.0}
\begin{center}
\caption{
Details of the real-world datasets in our experiments.}\label{tab:datasets}
\begin{tabular}{|l|r|r|r|r|r|r|}        \hline
Dataset  &a9a & mushrooms &  news20 & splice&svmguide3 & w8a\\
\hline\hline
$\#$ examples &48,842   & 8,124 & 16,242  & 3,175 &1,284 & 64,700\\
$\#$ features & 123  &  112  & 100  & 60 & 21 &  300
\\\hline
\end{tabular}
\end{center}
\renewcommand*\arraystretch{1.0}
\end{table}

To make a fair comparison, all algorithms adopt the same experimental setup. In particular, we set the penalty parameter $\gamma=\nu=1/n$, where $n$ is the number of training examples, and the trade-off parameter $\beta=1$. In addition,  we set the step size parameter $\eta_t=1/(\gamma t)$ for SADMM according to the theorem 2 in~\cite{ouyang2013stochastic}. Finally, the smooth parameter $a$ is set as $1$, and the step size for adaptive stochastic ADMM algorithms are searched from $2^{[-5:5]}$ using cross validation.

All the experiments were conducted with 5 different random seeds and 2 epochs ($2n$ iterations) for each dataset. All the result were reported by averaging over these 5 runs. We evaluated the learning performance by measuring objective values, i.e., $f(\u)$, and test error rates on the test datasets. In addition, we also evaluate computational efficiency of all the algorithms by their running time.  All experiments were run in Matlab over a machine of 3.4GHz CPU.

\subsection{Performance Evaluation}
The figure~\ref{fig:ADMM} shows the performance of all the algorithms in comparison over trials, from which we can draw several observations. Firstly, the left column shows the objective values of the three algorithms. We can observe that the two adaptive stochastic ADMM algorithms converge much faster than SADMM, which shows the effectiveness of exploration of adaptive (sub)gradient to accelerate stochastic ADMM. Secondly, compared with Ada-SADMM$_{diag}$, Ada-SADMM$_{full}$ achieves slightly smaller objective values on most of the datasets, which indicates full matrix is slightly more informative than the diagonal one. Thirdly, the central column provides test error rates of three algorithms, where we observe that the two adaptive algorithms achieve significantly smaller or comparable test error rates at $0.25$-th epoch  than SADMM at $2$-th epoch. This observation indicates that we can terminate the two adaptive algorithms earlier to save time and at the same time achieve similar performance compared with SADMM.   Finally, the right column shows the running time of three algorithms, which shows that during the learning process, the Ada-SADMM$_{full}$ is significantly slower while the Ada-SADMM$_{diag}$ is overall efficient compared with SADMM.  In summary, the Ada-SADMM$_{diag}$ algorithm achieves a good trade-off between the efficiency and effectiveness.

Table~\ref{tab:experiment} summarizes the performance of  all the compared algorithms over the 6 datasets, from which we can make similar observations. This again verifies the effectiveness of the proposed algorithms.

\begin{table}[htpb]
\renewcommand*\arraystretch{1.8}
\begin{center}
\caption{Evaluation of stochastic ADMM algorithms on the real-world  data sets.}\label{tab:experiment}
\begin{scriptsize}
\begin{tabular}{|l|c|c|c|c|c|c|c|}        \hline
\multirow{1}{*}{\footnotesize Algorithm}& \multicolumn{3}{|c|}{\bf a9a} & & \multicolumn{3}{|c|}{\bf mushrooms}\\
\cline{2-4} \cline{6-8}
 &   Objective value        &  Test error rate & ~Time (s)~ &&   Objective value        &  Test error rate & ~Time (s)~
\\
\hline\hline
SADMM	&	2.6002 	$\pm$	0.4271 	&	0.1646 	$\pm$	0.0075 	&	56.0914 	&&	0.7353 	$\pm$	0.2104 	&	0.0350 	$\pm$	0.0136 	&	7.6619 	\\
Ada-SADMM$_{diag}$	&	0.3550 	$\pm$	0.0001 	&	0.1501 	$\pm$	0.0012 	&	94.7619 	&&	0.0096 	$\pm$	0.0005 	&	0.0006 	$\pm$	0.0000 	&	13.0355 	\\
Ada-SADMM$_{full}$	&	0.3545 	$\pm$	0.0001 	&	0.1498 	$\pm$	0.0013 	&	622.4459 	&&	0.0091 	$\pm$	0.0002 	&	0.0002 	$\pm$	0.0003 	&	67.8198 	
\\\hline\hline
\multirow{1}{*}{\footnotesize Algorithm}& \multicolumn{3}{|c|}{\bf news20} & & \multicolumn{3}{|c|}{\bf splice}\\
\cline{2-4} \cline{6-8}
 &   Objective value        &  Test error rate & ~Time (s)~ &&   Objective value        &  Test error rate & ~Time (s)~
\\
\hline\hline
SADMM	&	0.5652 	$\pm$	0.0151 	&	0.1333 	$\pm$	0.0034 	&	13.2948 	&&	108.6823 	$\pm$	20.9655 	&	0.2454 	$\pm$	0.0322 	&	0.9821 	\\
Ada-SADMM$_{diag}$	&	0.3139 	$\pm$	0.0003 	&	0.1280 	$\pm$	0.0015 	&	22.4788 	&&	0.3793 	$\pm$	0.0054 	&	0.1578 	$\pm$	0.0059 	&	1.3674 	\\
Ada-SADMM$_{full}$	&	0.3204 	$\pm$	0.0007 	&	0.1284 	$\pm$	0.0016 	&	148.5242 	&&	0.3710 	$\pm$	0.0014 	&	0.1550 	$\pm$	0.0079 	&	7.0392 	
\\\hline\hline
\multirow{1}{*}{\footnotesize Algorithm}& \multicolumn{3}{|c|}{\bf svmguide3} & & \multicolumn{3}{|c|}{\bf w8a}\\
\cline{2-4} \cline{6-8}
  &   Objective value        &  Test error rate & ~Time (s)~ &&   Objective value        &  Test error rate & ~Time (s)~
\\
\hline\hline
SADMM	&	1.6143 	$\pm$	0.3123 	&	0.2161 	$\pm$	0.0052 	&	0.1288 	&&	0.3357 	$\pm$	0.0916 	&	0.0957 	$\pm$	0.0012 	&	191.7544 	\\
Ada-SADMM$_{diag}$	&	0.5163 	$\pm$	0.0046 	&	0.2056 	$\pm$	0.0060 	&	0.2014 	&&	0.1526 	$\pm$	0.0010 	&	0.0931 	$\pm$	0.0005 	&	326.1392 	\\
Ada-SADMM$_{full}$	&	0.5230 	$\pm$	0.0044 	&	0.2000 	$\pm$	0.0044 	&	0.4602 	&&	0.1469 	$\pm$	0.0006 	&	0.0929 	$\pm$	0.0003 	&	4027.1963 	
\\\hline
\end{tabular}
\end{scriptsize}
\end{center}
\end{table}

\begin{figure}[htpb]
\begin{center}
\mbox{
\includegraphics[width=2in,height=1.35in]{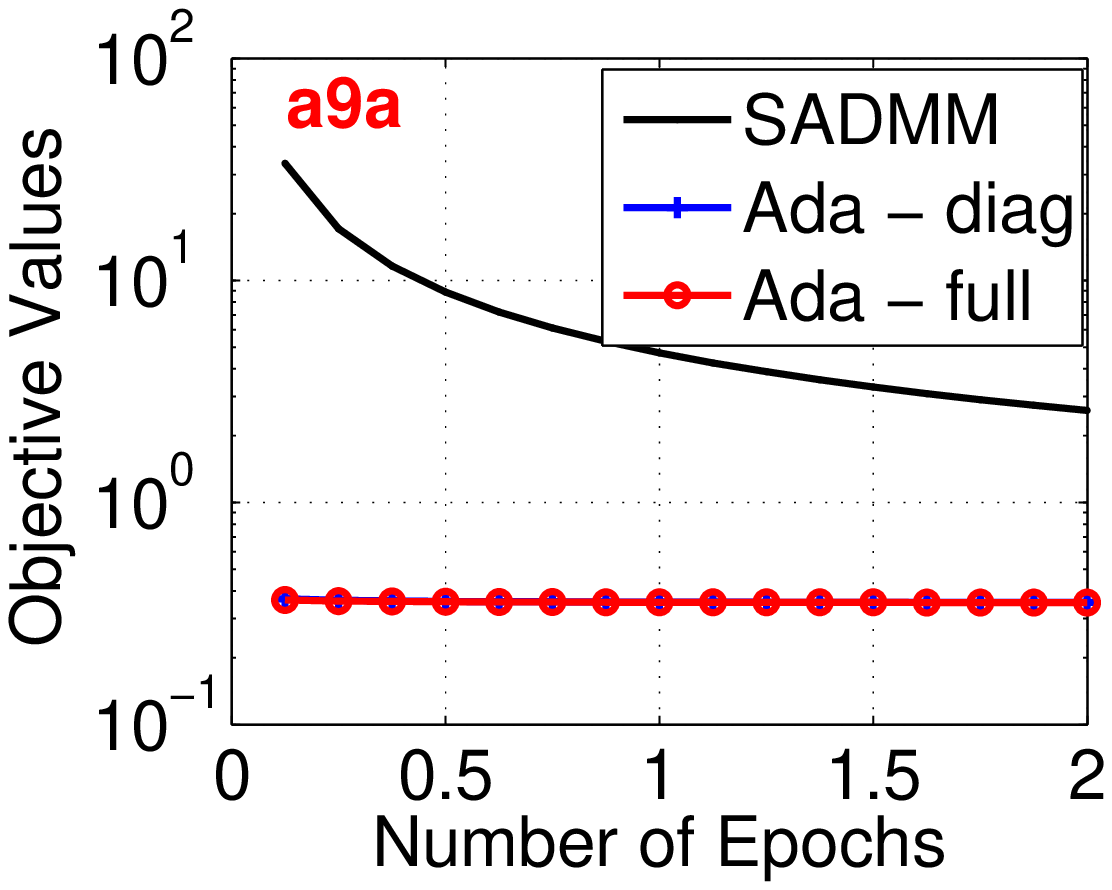}
\includegraphics[width=2in,height=1.35in]{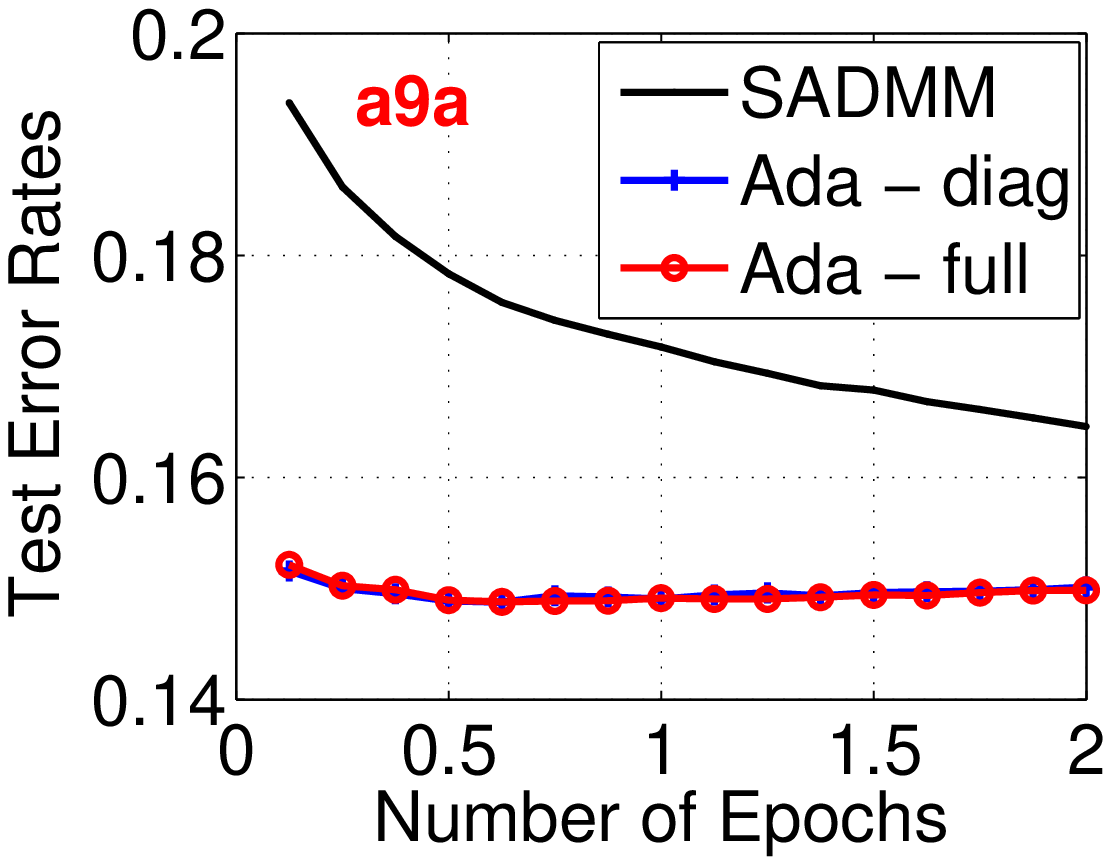}
\includegraphics[width=2in,height=1.35in]{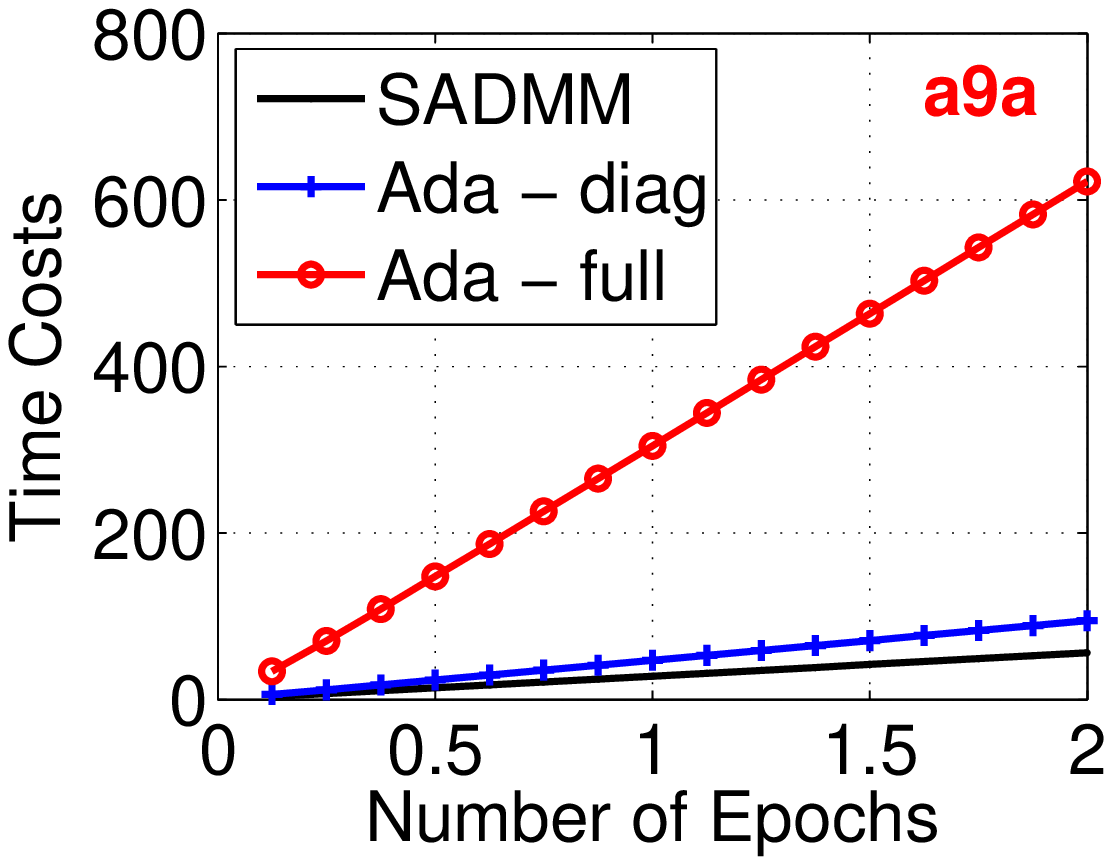}
}
\mbox{
\includegraphics[width=2in,height=1.35in]{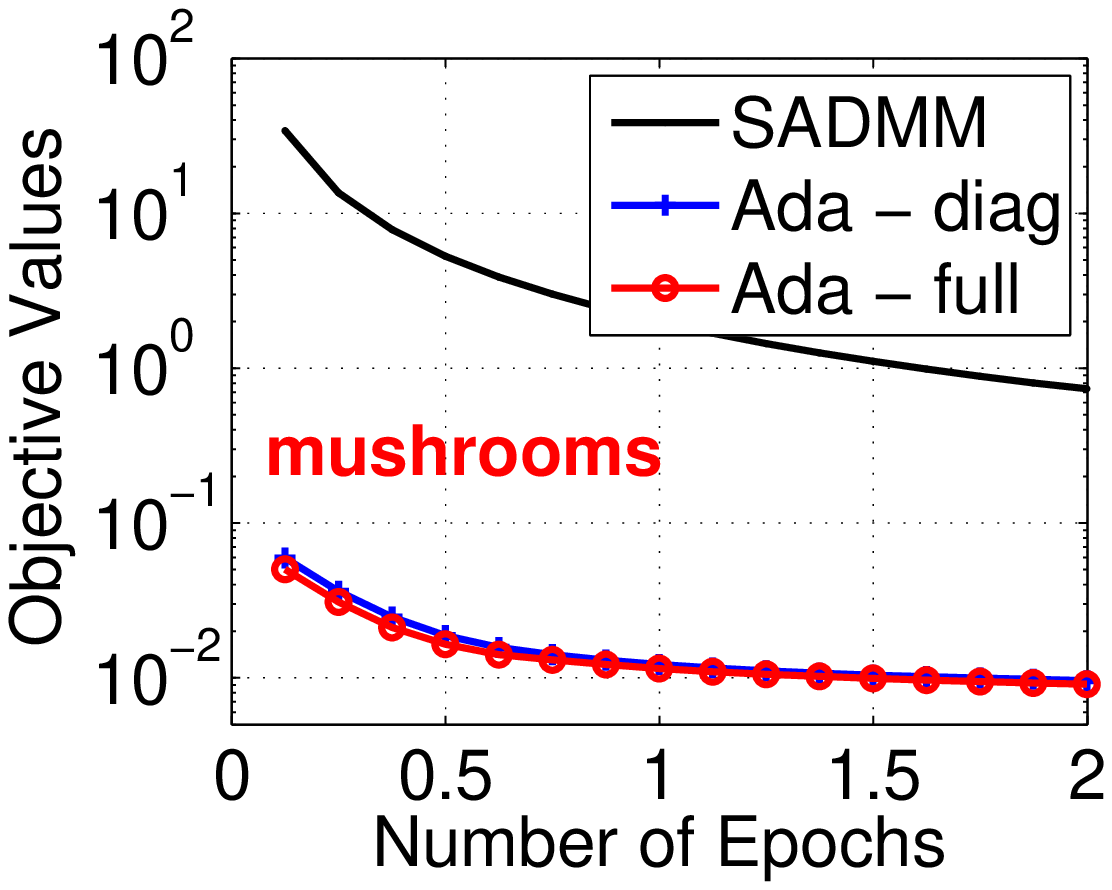}
\includegraphics[width=2in,height=1.35in]{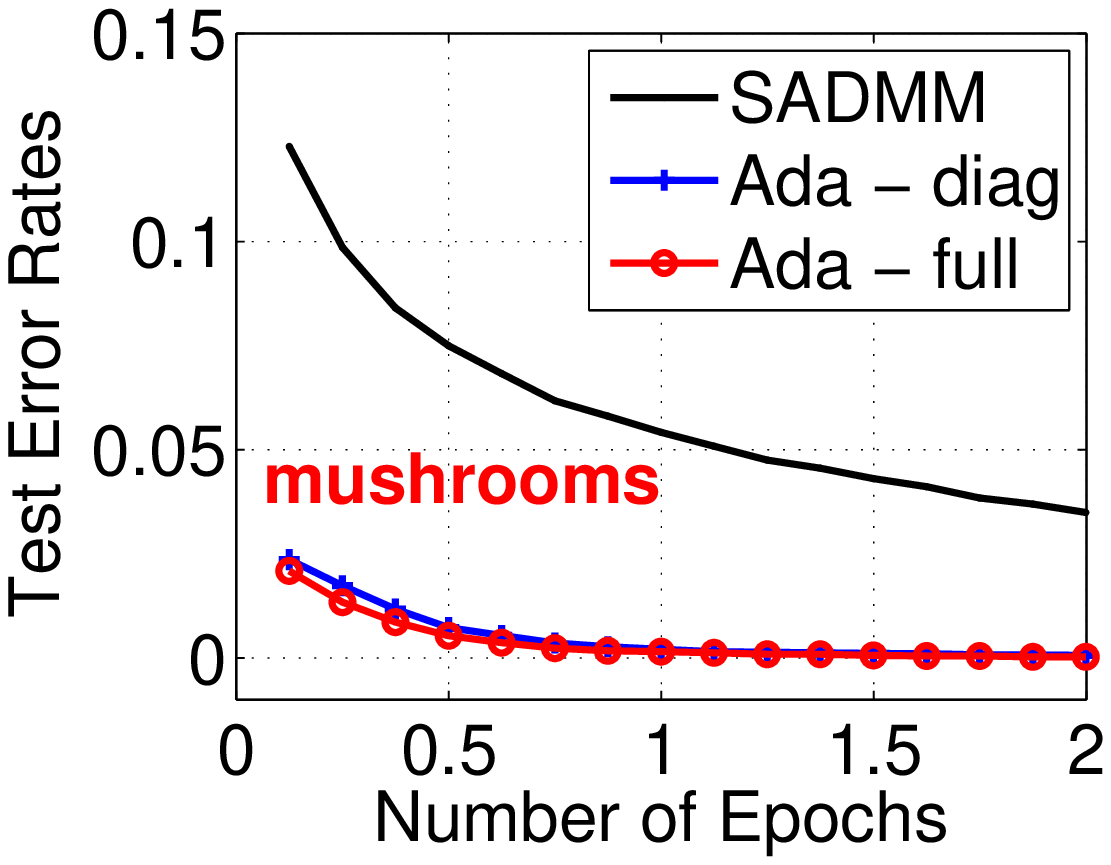}
\includegraphics[width=2in,height=1.35in]{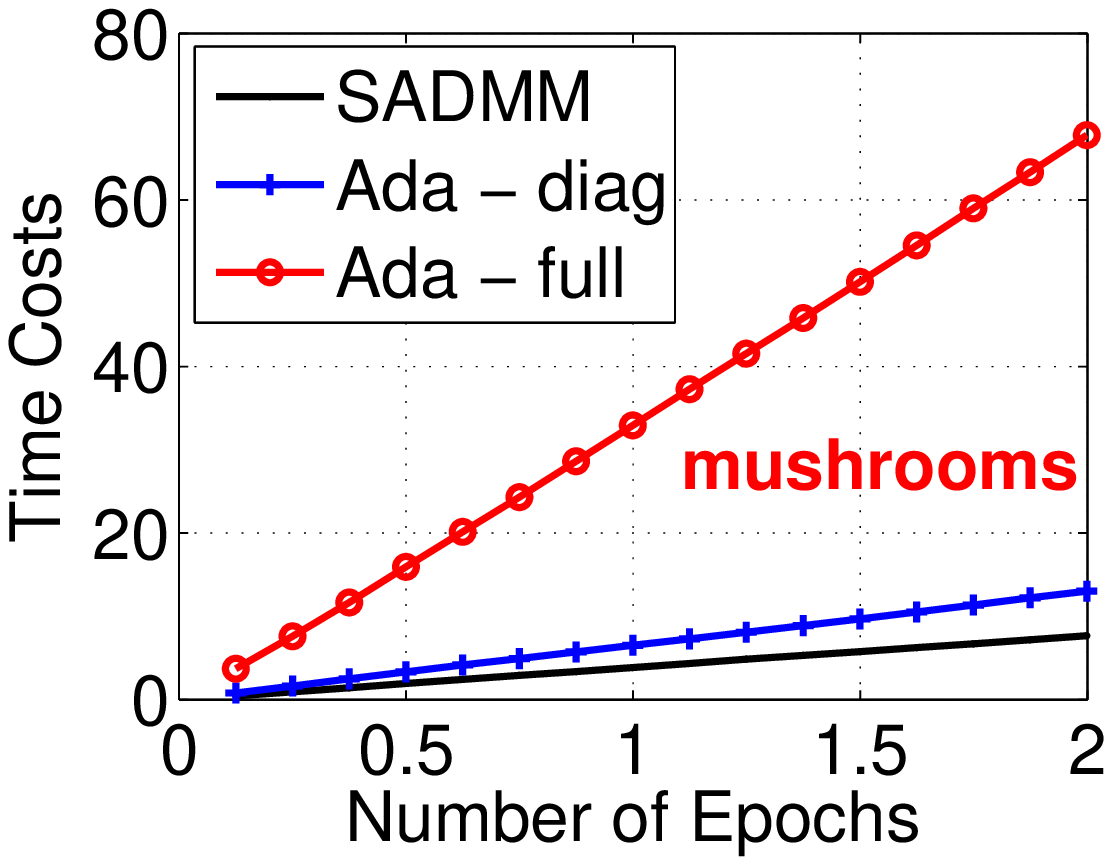}
}
\mbox{
\includegraphics[width=2in,height=1.35in]{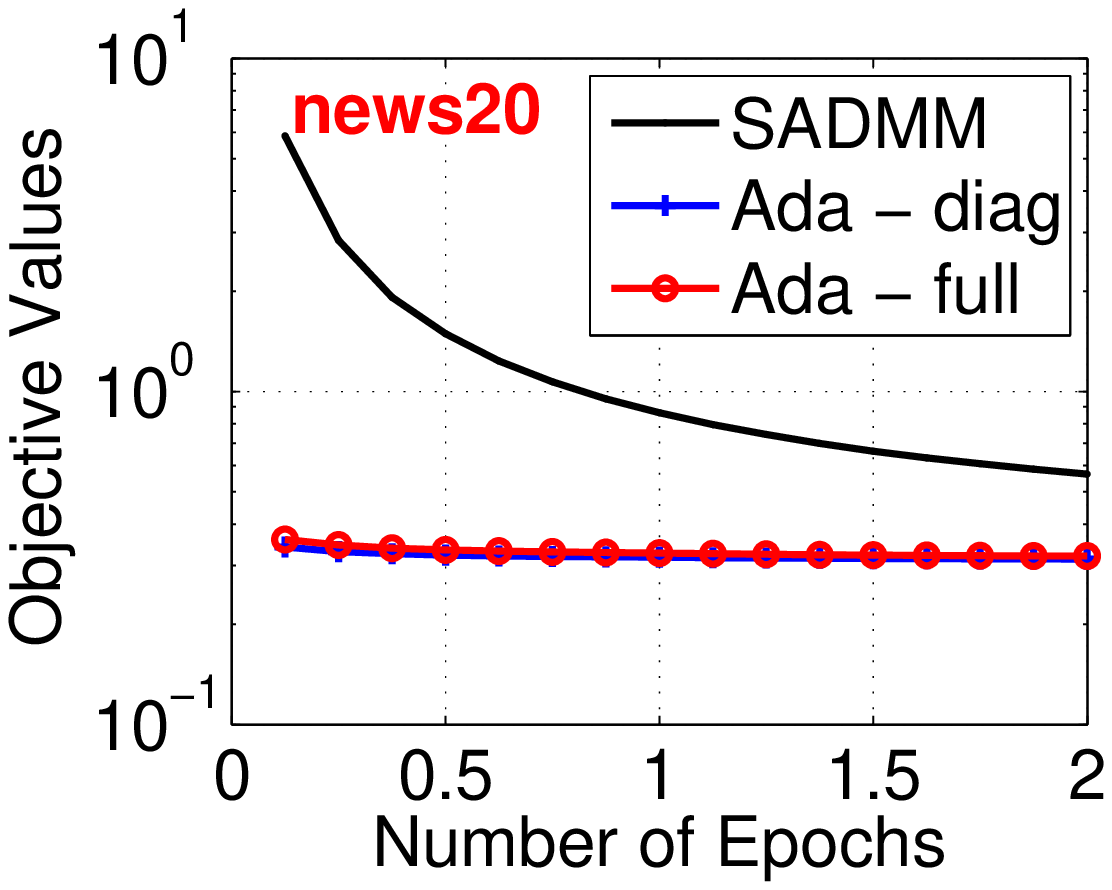}
\includegraphics[width=2in,height=1.35in]{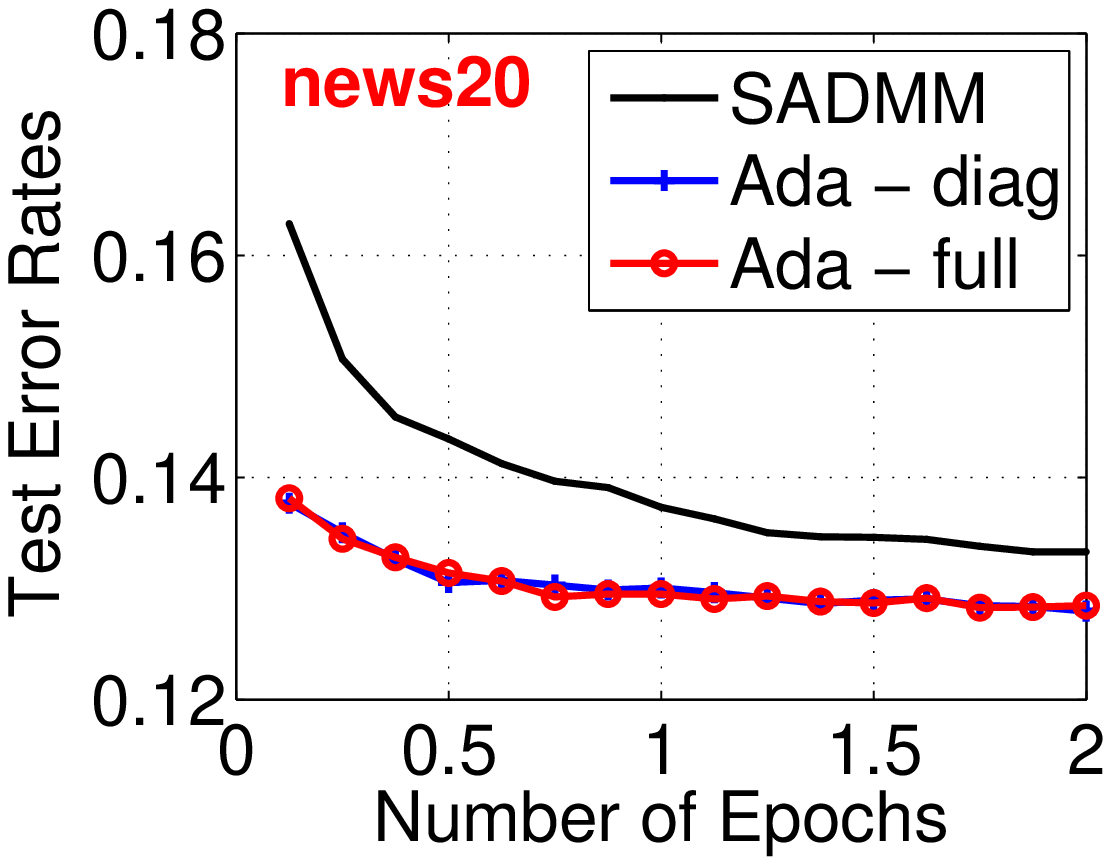}
\includegraphics[width=2in,height=1.35in]{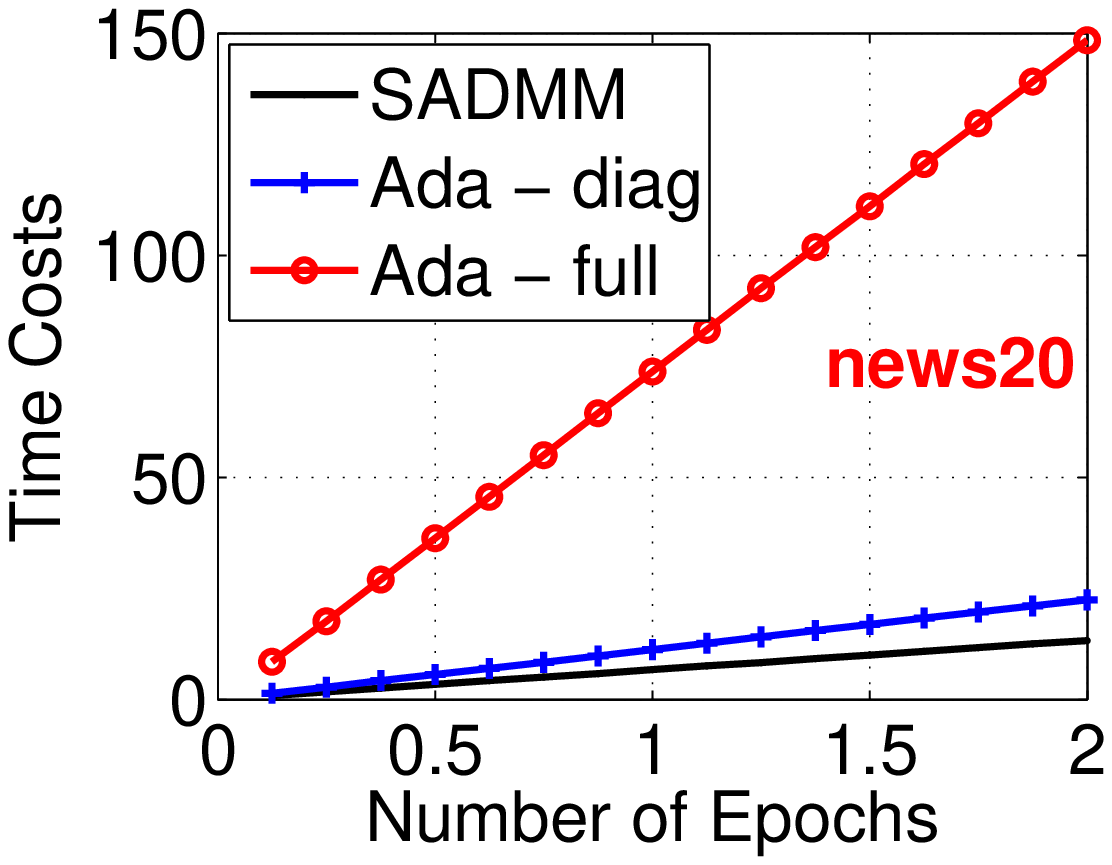}
}
\mbox{
\includegraphics[width=2in,height=1.35in]{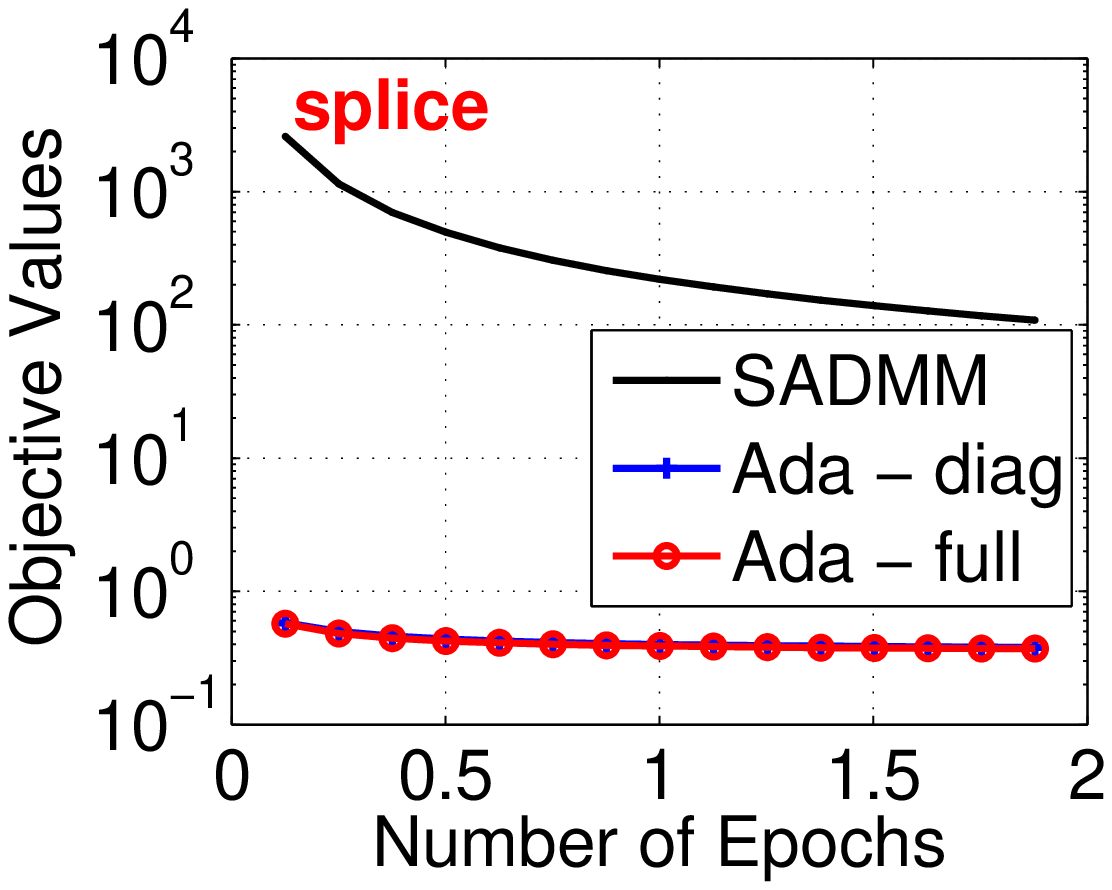}
\includegraphics[width=2in,height=1.35in]{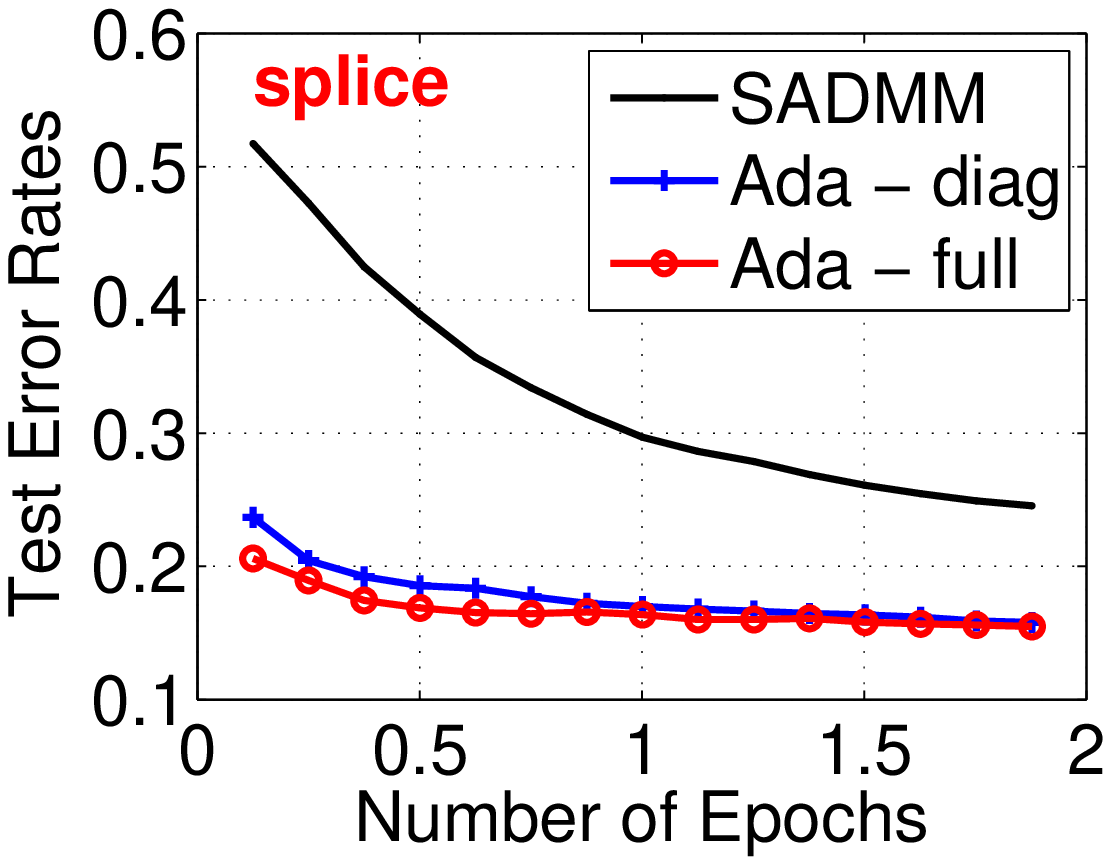}
\includegraphics[width=2in,height=1.35in]{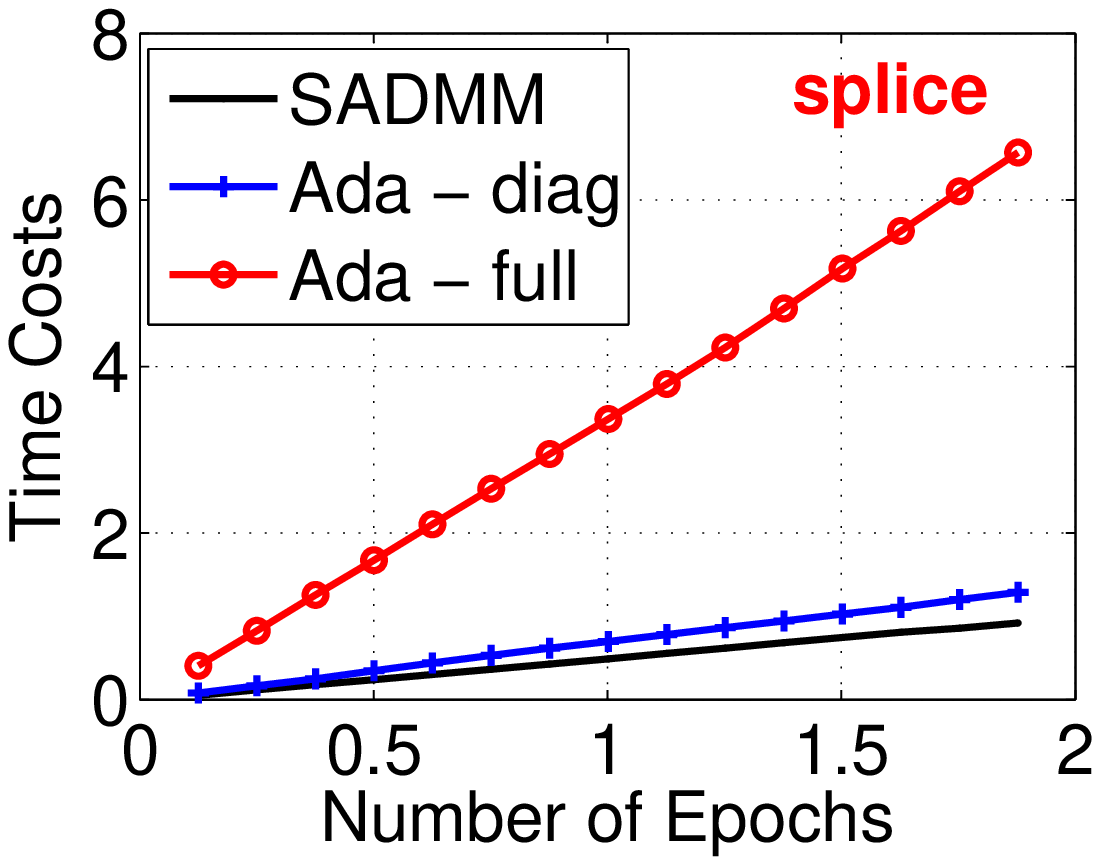}
}
\mbox{
\includegraphics[width=2in,height=1.35in]{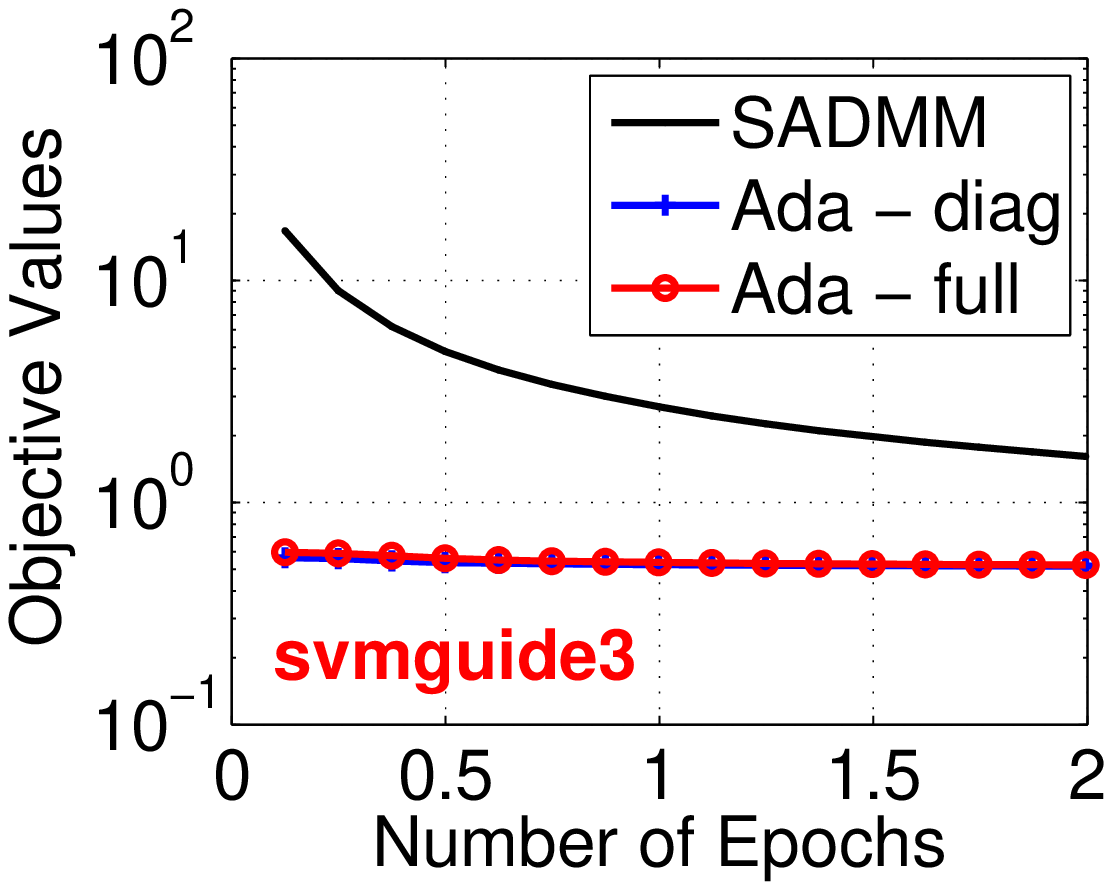}
\includegraphics[width=2in,height=1.35in]{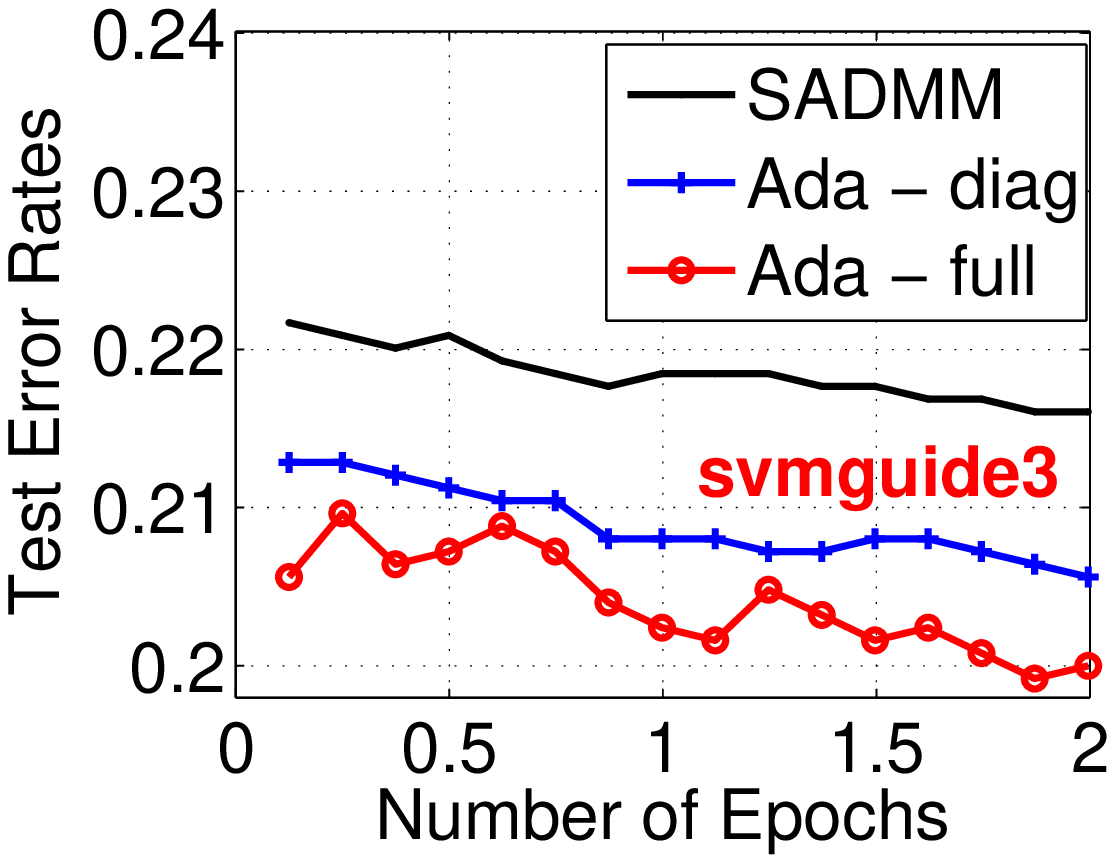}
\includegraphics[width=2in,height=1.35in]{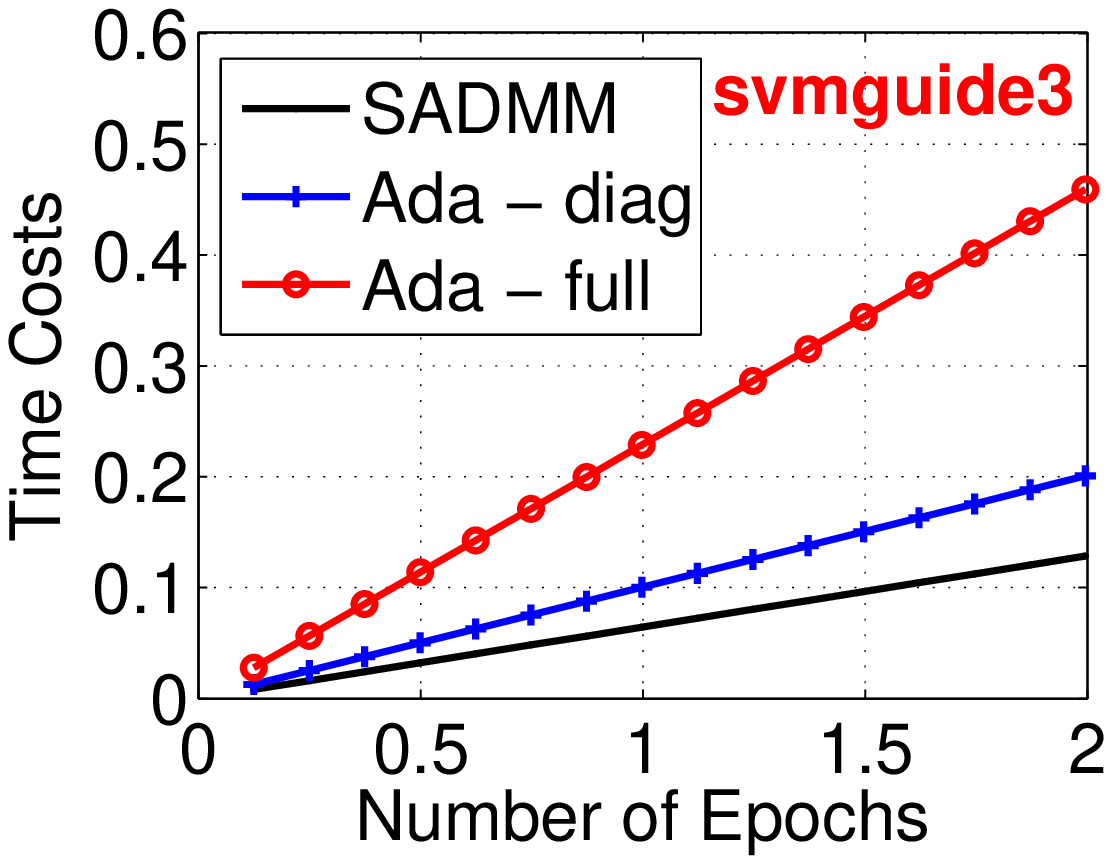}
}
\mbox{
\includegraphics[width=2in,height=1.35in]{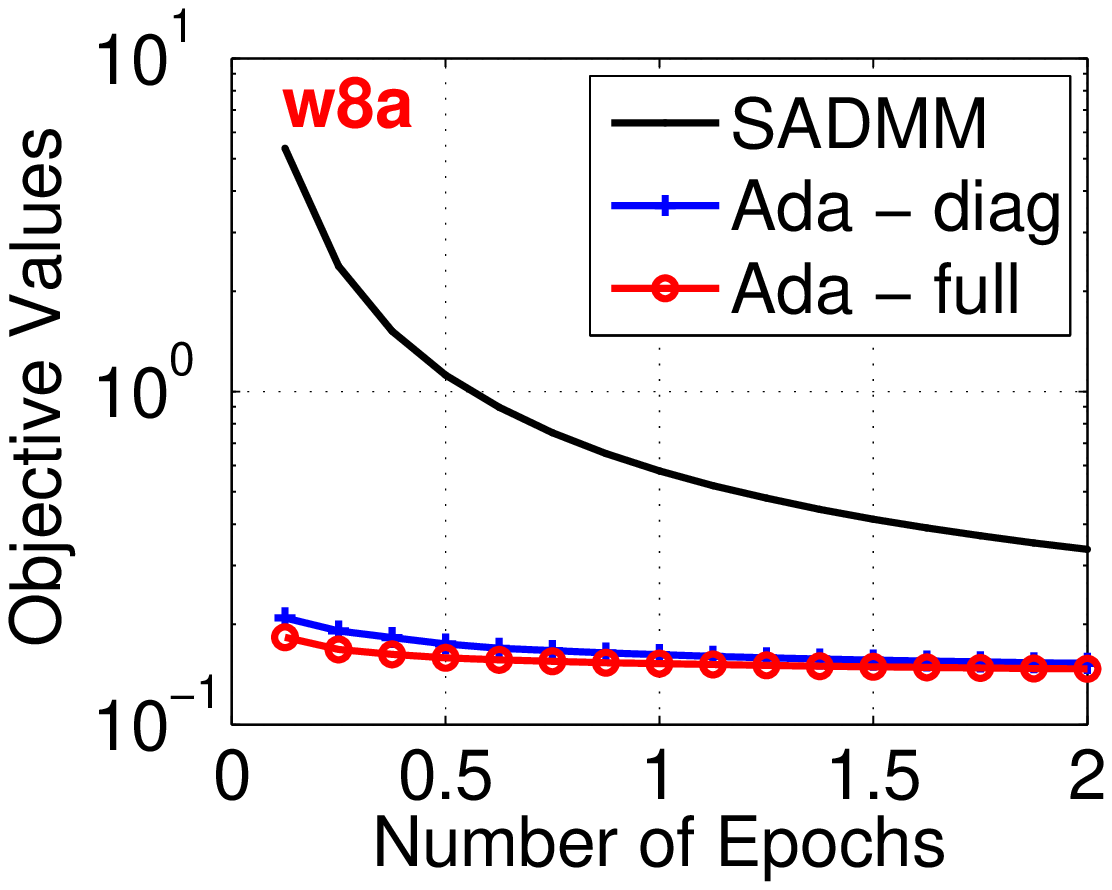}
\includegraphics[width=2in,height=1.35in]{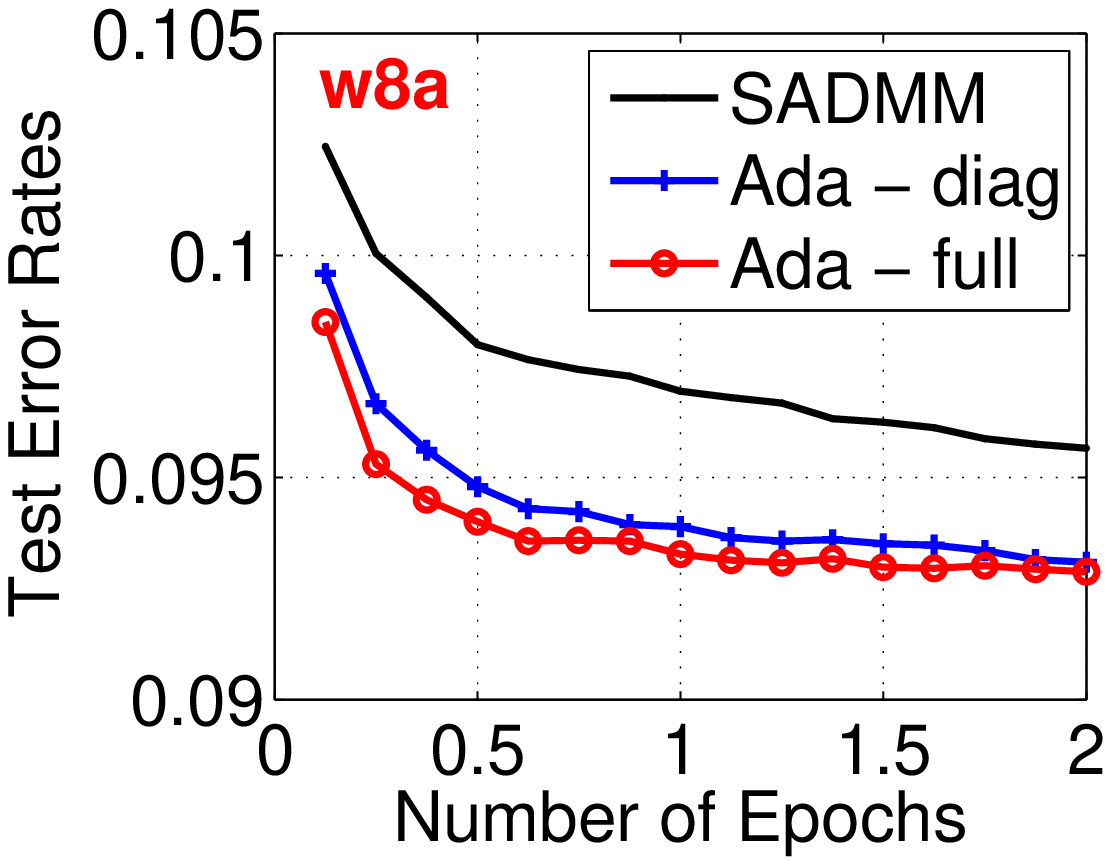}
\includegraphics[width=2in,height=1.35in]{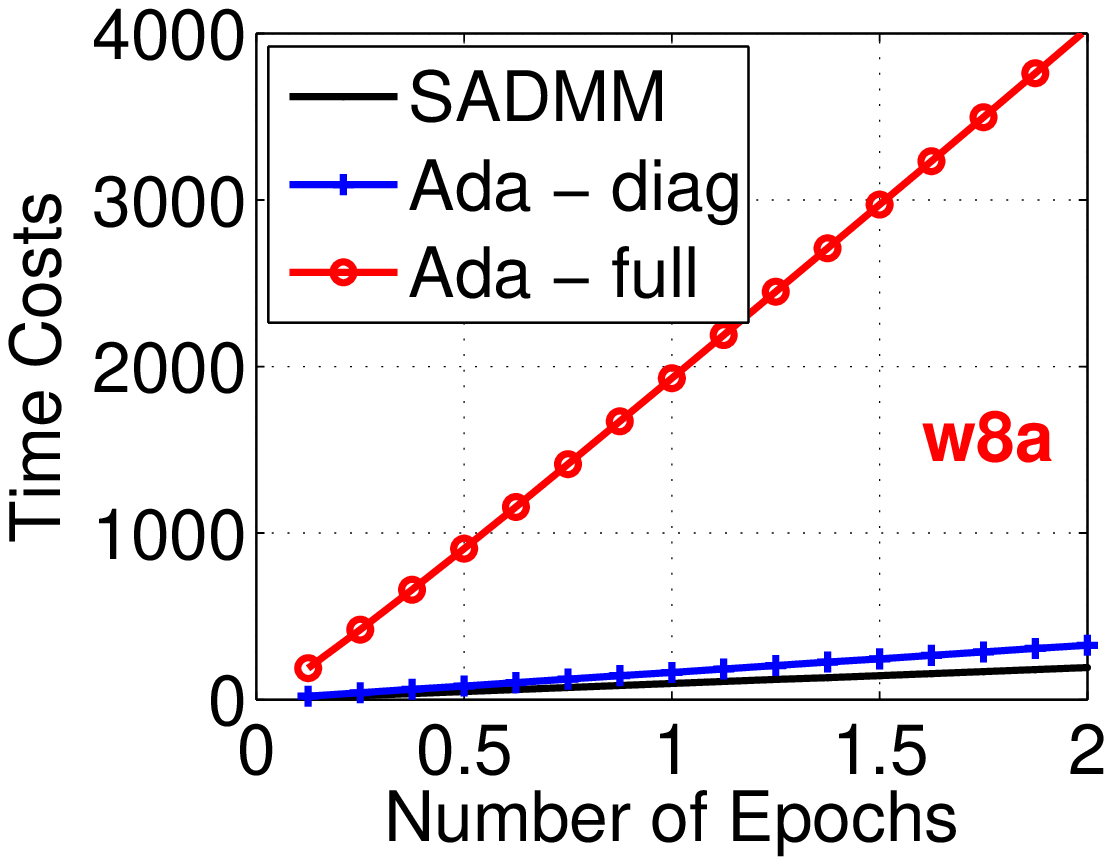}
}
\end{center}
\vspace{-0.2in}
\caption{Comparison between SADMM with  Ada-SADMM$_{diag}$ (``Ada-diag'') and Ada-SADMM$_{full}$ (``Ada-full'') on  6 real-world datasets. Epoch for the horizontal axis is the number of iterations divided by dataset size. \textbf{Left Panels}: Average objective values. \textbf{Middle Panels}:  Average test error rates. \textbf{Right Panels}: Average time costs (in seconds). }
\label{fig:ADMM}
\end{figure}

\section{Conclusion}
ADMM is a popular technique in machine learning. This paper studied to accelerate stochastic ADMM with adaptive subgradient, by replacing the fixed proximal function with adaptive proximal function. Compared with traditional stochastic ADMM, we show that the proposed adaptive algorithms converge significantly faster through the proposed adaptive strategies.  Promising experimental results on a variety of real-world datasets further validate the effectiveness of  our techniques.

\newpage
{
\bibliography{reference}
\bibliographystyle{plain}
}

\end{document}